\documentclass[10pt,twocolumn,letterpaper]{article}

\usepackage[pagenumbers]{configurations/cvpr} 

\usepackage{url}
\usepackage{tabularx} 

\usepackage{graphicx}
\usepackage{subcaption}
\renewcommand{\paragraph}[1]{\noindent\textbf{#1}}

\usepackage{enumitem}
\usepackage{soul}
\usepackage{multirow}
\usepackage{xcolor}
\usepackage{tikz}
\usetikzlibrary{bayesnet}
\usetikzlibrary{arrows}
\usepackage{caption}
\usetikzlibrary{backgrounds}

\usepackage{wrapfig}

\usepackage{algorithmic}
\usepackage[ruled,vlined,noend,linesnumbered]{algorithm2e}
\SetKwRepeat{Do}{do}{while}

\usetikzlibrary{arrows.meta,positioning,bending}

\usepackage{amsmath,amssymb,amsthm}
\usepackage{thmtools,thm-restate}
\usepackage{mathtools}

\theoremstyle{plain}
\newtheorem{theorem}{Theorem}[section]

\theoremstyle{definition}
\newtheorem{definition}[theorem]{Definition}
\theoremstyle{remark}

\makeatletter
\@namedef{ver@everyshi.sty}{}
\makeatother
\usepackage{tikz}
\usetikzlibrary{bayesnet,arrows,backgrounds}

\usepackage{tcolorbox}

\usepackage{graphicx}
\usepackage{subcaption}
\renewcommand{\paragraph}[1]{\noindent\textbf{#1}}

\newcommand{\stkout}[1]{\ifmmode\text{\sout{\ensuremath{#1}}}\else\sout{#1}\fi}

\usepackage{amsmath,amsfonts,bm}

\def\eqref#1{equation~\ref{#1}}

\def\1{\bm{1}}

\providecommand{\1}{\mathbf{1}}

\providecommand{\dd}{\mathbf{d}}

\providecommand{\uu}{\mathbf{u}}

\providecommand{\ww}{\mathbf{w}}
\providecommand{\xx}{\mathbf{x}}
\providecommand{\yy}{\mathbf{y}}
\providecommand{\zz}{\mathbf{z}}

\providecommand{\cE}{\mathcal{E}}

\providecommand{\cG}{\mathcal{G}}

\providecommand{\cL}{\mathcal{L}}

\providecommand{\cS}{\mathcal{S}}

\providecommand{\cV}{\mathcal{V}}

\providecommand{\cZ}{\mathcal{Z}}

\providecommand{\R}{\mathbb{R}} 

\providecommand{\mA}{\mathbf{A}}

\providecommand{\mH}{\mathbf{H}}

\newcommand{\norm}[1]{\left\lVert#1\right\rVert}

\usepackage{amsmath,amssymb,amsthm}
\usepackage{thmtools,thm-restate}
\usepackage{mathtools}

\usepackage{enumitem}
\usepackage{soul}
\usepackage{multirow}
\usepackage{xcolor}
\usepackage{caption}
\usepackage{wrapfig}

\providecommand{\obse}{x}
\providecommand{\disc}{d}
\providecommand{\late}{z}

\providecommand{\obs}{\xx}
\providecommand{\dis}{\dd}
\providecommand{\sdis}{\tilde{\dis}}
\providecommand{\lat}{\zz}

\providecommand{\Zz}{\cZ}

\providecommand{\Ds}{\Omega_{\mathrm{supp}}}
\providecommand{\Dc}{\Omega_{\mathrm{CP}}}
\providecommand{\Dcomp}{\Omega_{\mathrm{comp}}}

\providecommand{\parents}[1]{\mathrm{Pa}(#1)}

\providecommand{\supp}[1]{\mathrm{supp}(#1)}

\newcommand{\ours}{\textbf{HierDiff}\xspace}

\definecolor{cvprblue}{rgb}{0.21,0.49,0.74}
\usepackage[pagebackref,breaklinks,colorlinks,allcolors=cvprblue]{hyperref}

\def\paperID{} 
\def\confName{}
\def\confYear{}

\title{Learning by Analogy: A Causal Framework for Composition Generalization}

\author{
\textbf{Lingjing Kong}$^{1}$, \textbf{Shaoan Xie}$^{1}$, \textbf{Yang Jiao}$^{2}$, \textbf{Yetian Chen}$^{2}$, \textbf{Yanhui Guo}$^{2}$, \\ 
\textbf{Simone Shao}$^{2}$, \textbf{Yan Gao}$^{2}$, \textbf{Guangyi Chen}$^{1,3}$, \textbf{Kun Zhang}$^{1,3}$ \\[2mm]
$^{1}$Carnegie Mellon University \quad
$^{2}$Amazon \quad
$^{3}$Mohamed bin Zayed University of Artificial Intelligence
}

\begin{document}
\maketitle

\begin{abstract}
    Compositional generalization -- the ability to understand and generate novel combinations of learned concepts -- enables models to extend their capabilities beyond limited experiences.
    While effective, the data structures and principles that enable this crucial capability remain poorly understood.
    We propose that compositional generalization fundamentally requires decomposing high-level concepts into basic, low-level concepts that can be recombined across similar contexts, similar to how humans draw \textbf{analogies} between concepts. For example, someone who has never seen a peacock eating rice can envision this scene by relating it to their previous observations of a chicken eating rice.
    In this work, we formalize these intuitive processes using principles of \textbf{causal modularity} and \textbf{minimal changes}. 
    We introduce a hierarchical data-generating process that naturally encodes different levels of concepts and their interaction mechanisms.
    Theoretically, we demonstrate that this approach enables compositional generalization supporting complex relations between composed concepts, advancing beyond prior work that assumes simpler interactions like additive effects.
    Critically, we also prove that this latent hierarchical structure is \textbf{provably recoverable} (identifiable) from observable data like text-image pairs, a necessary step for learning such a generative process. To validate our theory, we apply insights from our theoretical framework and achieve significant improvements on benchmark datasets.
\end{abstract}

\section{Introduction}
Compositional generalization is a hallmark of human intelligence, enabling us to navigate a vast array of novel situations despite limited direct experience.
Humans often achieve this by \textbf{drawing analogies}~\citep{gentner1983structure,holyoak1989analogical}.
Even without having seen a peacock eating rice, one can relate this scene to observations of a chicken eating rice by recognizing shared low-level features: both peacocks and chickens have beaks, wings, and other common attributes. This process effectively transforms interactions between high-level concepts (``peacock \& rice'') into more fundamental low-level concepts (``beak \& rice'' or ``peck'' as shown in Figure~\ref{fig:teaser}). Since these low-level, elementary interactions appear across many observed scenarios, they can be transferred to accurately visualize novel combinations of high-level concepts.
This process involves two key cognitive steps: 1) \emph{decomposing complex, high-level concepts into low-level, modular components}, and 2) \emph{recombining these low-level concepts to synthesize the novel scene}.
Clearly, this ability depends on favorable structures in the underlying data distribution.
In light of this, we aim to address the following fundamental question:

\begin{center}
\textit{What latent data structures entail this analogical, modular process to enable compositional generalization, and can we provably learn them from observational data?}
\end{center}
Answering this question is essential for deliberately incorporating this valuable capability into machine learning models, which typically perform poorly when confronted with data outside their training distribution \citep{koh2021wilds, recht2019imagenet, taori2020measuring}.
\begin{figure*}[t]
    \centering
    \begin{minipage}{0.6\linewidth}
        \centering
        \includegraphics[width=0.8\linewidth]{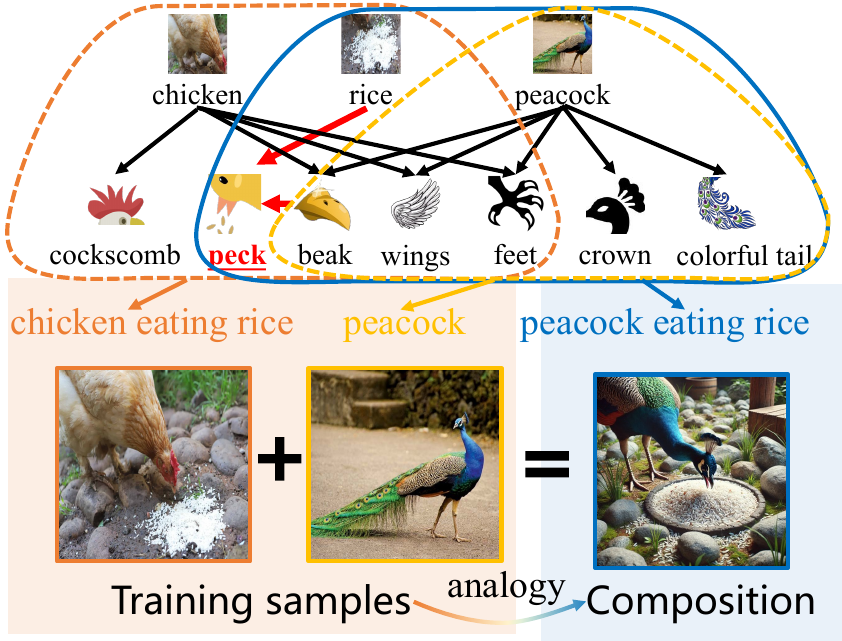}
        \caption{
            \textbf{Composition by analogy.}
            Dotted circles group concepts by their sources.
            The hierarchical process generalizes to an unseen concept \textcolor{blue}{``peacock eating rice''} by composing low-level modules \textcolor{yellow!36!orange}{``peacock''} and \textcolor{orange}{``chicken eating rice''}. The interaction ``beak \& rice'' (denoted as \textcolor{red}{``peck''}) transfers from \textcolor{orange}{``chicken eating rice''} to \textcolor{blue}{``peacock eating rice''}, akin to human making analogies.
            The resulting composition \textcolor{blue}{``peacock eating rice''} inherits modular components from all source concepts.
        }
        \label{fig:teaser}
    \end{minipage}
    \hfill
    \begin{minipage}{0.36\linewidth}
        \centering
        \includegraphics[width=1\linewidth]{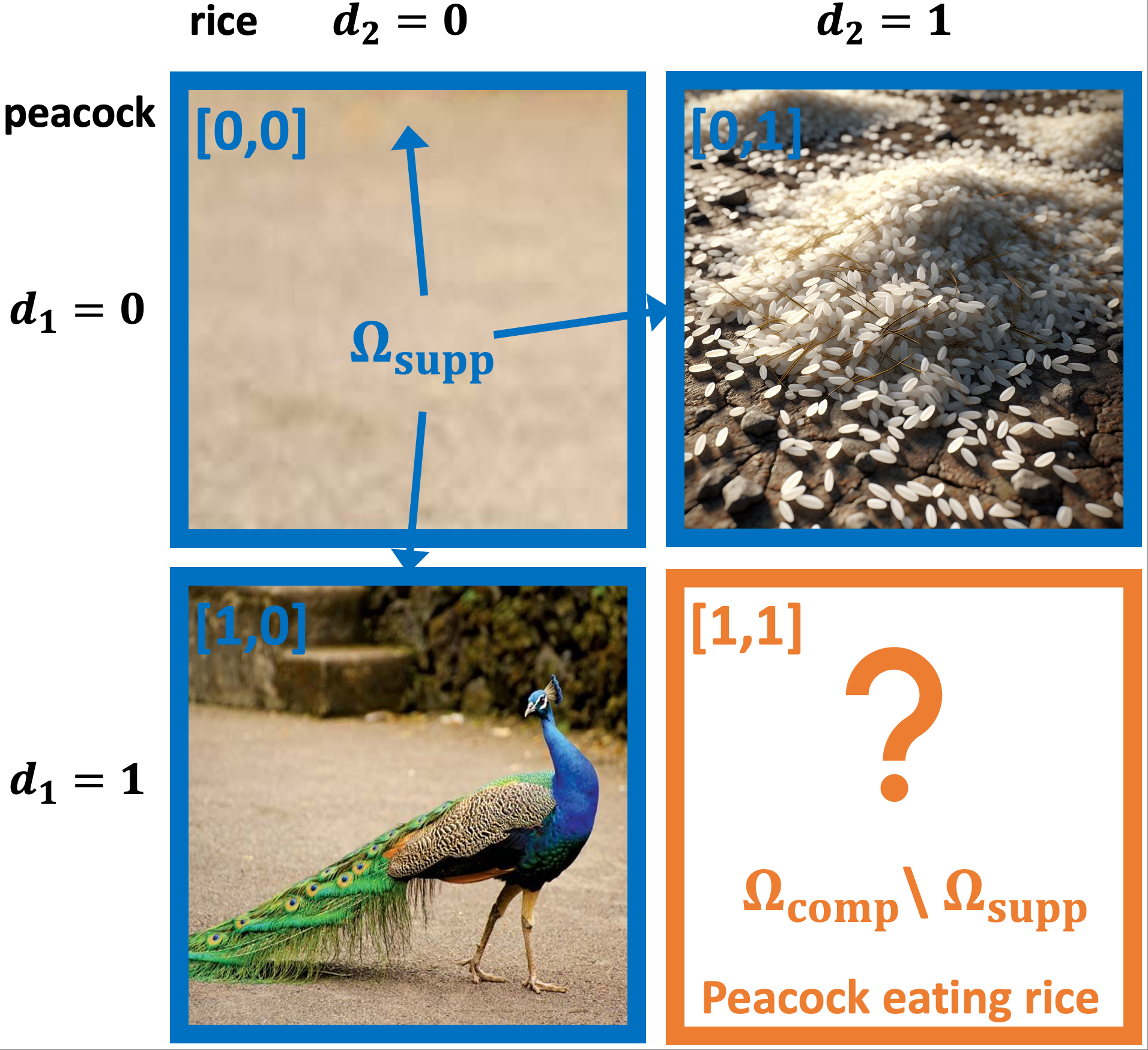}
        \caption{
            \textbf{Compositional generalization.}
            The training support $\Ds$ only contains ``peacock'' and ``rice'' separately. Compositional generalization aims to generate ``peacock eating rice'' in the out-of-support region $\Dcomp \setminus \Ds$. $d_1$, $d_2$ indicate the presences of ``peacock'' and ``rice'' (e.g., $d_1=1$ for ``peacock'', $d_1=0$ for its absence).
        }
        \label{fig:problem_motivation}
    \end{minipage}
    \vspace{-0.3cm}
\end{figure*}

Despite substantial empirical advances~\citep{ramesh2022hierarchical, du2024compositional, liu2022compositional, zhang2024realcompo, hu2024ella, huang2023t2icompbench, bar2023multidiffusion, yang2024mastering, trager2023linear, uselis2025does, okawa2023compositional}, theoretical understanding remains limited and relies on restrictive assumptions about concept interactions.
For instance, \citet{brady2023provably} and \citet{wiedemer2024provable} assume concepts affect separate pixel regions without interaction, while \citet{lachapelle2023additive} proposes additive concept influences in pixel space, later extended by \citet{brady2024interaction} to include second-order polynomial terms. 
Importantly, these contributions generally overlook the hierarchical nature of concepts and cannot capture the complex, nonparametric interactions (like ``beak'' \& ``rice'') that are learned and transferred in analogical reasoning


We formulate the intuition of analogy making with the language of causality, specifically, a latent \textbf{hierarchical} model that encodes fundamental principles of \textbf{modularity} and \textbf{minimal changes} in causality~\citep{pearl2009causal,spirtes2000causation,peters2017elements}.
Causal modularity (or ``invariant mechanism'') posits that a complex system, like image generation, can be decomposed into distinct modules (e.g., ``beak \& rice'') that are autonomous and transferable across different contexts. The minimal-change principle formalizes the idea that an observed difference should result from the fewest possible underlying changes (e.g., ``crown'' in ``chicken'' and ``peacock'' in Figure~\ref{fig:teaser}), making knowledge transfer possible.
Our framework accommodates complex interactions among high-level concepts and their intricate relations in the latent space, beyond prior work~\citep{wiedemer2024provable,lachapelle2023additive,brady2024interaction,brady2023provably}.

We establish identification conditions that allow latent hierarchical models to be learned from observed data, e.g., text-image pairs ubiquitous in image generation tasks. Unlike previous work on identifying latent hierarchical models~\citep{huang2022latent, choi2011learning, anandkumar2013learning, pearl1994probabilistic, dong2023versatile}, our theory does not require linearity or discrete latent variables, thus capable of modeling more complex data distributions. While recent work~\citep{kong2023identification} also addresses nonlinear hierarchical models and identifies concepts in groups, our approach leverages interactions among latent variables to identify \textbf{individual latent concepts} and the \textbf{graphical structure}.
Building on this theoretical foundation, we demonstrate how the abstract concepts of hierarchical levels and modularity can be practically realized by interpreting diffusion timesteps as hierarchical levels and by enforcing an explicit sparsity regularizer on concept attention maps. Our empirical results validate that integrating these theoretically motivated design choices leads to significant improvements in compositional generation.

Please refer to Appendix~\ref{app:related_work} for a more detailed discussion on related work.

\section{Compositional Generalization and Hierarchical Models}
We formally define compositional generalization and introduce the latent hierarchical data-generating process, which lies at the core of our framework. We denote the dimensionality of a multidimensional variable with $n(\cdot)$, the integer set $\{i \}_{i=1}^{n}$ with $[n]$, and all parents of $v$ with $\parents{v}$.

\paragraph{Compositional generalization.}
Let $ \obs $ denote the observed variables $ \obs \in \R^{d_{\obs}} $ of interest (e.g., natural images).
Let $\dis \in \{0, 1, \dots\}^{n(\dis)}$ be discrete variables that control high-level concepts present in the paired data $\obs$ (e.g., ``peacock'').
Then, text-to-image generation entails learning the condition distribution $ p( \obs | \dis ) $, where we specify the discrete concepts $\dis$ through text to generate the corresponding image $\obs$.
However, the training data distribution often lacks data containing certain combinations of concepts, even when each concept appears separately.
In Figure~\ref{fig:problem_motivation}, $d_{1}$ and $d_{2}$ indicate the presence of ``peacock'' and ``rice'' when they take value $1$.
Although we may observe ``peacock'' and ``rice'' in separate images (i.e., the training distribution contains images with $ \dis = [0, 1] $ and $ \dis = [1, 0] $), their composition $ \dis = [1, 1] $ required to produce ``peacock eating rice'' may be absent from the training support $\Ds$.
Since we can only train the model $\hat{p} (\obs | \dis)$ to match the true distribution $ p(\obs | \dis) $ over the support $ \{[0, 0], [0, 1], [1, 0]\} $, the model $\hat{p} (\obs | \dis) $ might produce arbitrary results for the out-of-support input $ \dis = [1, 1] $.
In this context, \emph{compositional generalization} refers to when our model $\hat{p} (\obs | \dis)$, which agrees with the true model $ p(\obs | \dis) $ on the support $ \Ds $, agrees on a strictly larger space $ \Ds \subset \Dcomp $.
We call a set of concepts $\dis$ \textit{composable} if it lies within the compositional space $ \dis \in \Dcomp $.
An important example of $ \Dcomp $ is the Cartesian product space $ \Dc := [\Ds]_{1} \times \dots \times [\Ds]_{n(\dd)}$~\citep{lachapelle2023additive,wiedemer2024compositional} where $ [\Ds]_{i}: = \{ \disc_{i}: \disc \in \Ds \} $ denotes the marginal support of dimension $i$.
In this case, the model should correctly compose concepts that appear separately in the training.


\paragraph{Challenges and motivations.}
Recent work in causal representation learning has increasingly focused on establishing provable conditions for compositional generalization. To address the extrapolation challenge,
prior work~\citep{lachapelle2023additive,wiedemer2024compositional,wiedemer2024provable} proposes additive generating functions, where the joint influence of multiple latent concepts $\lat_{i}$ can be expressed as the sum of their individual influences $ \obs: = \sum_{i} g_{i} (\lat_{i}) $.
While this semi-parametric approach offers certain compositional properties, it fails to adequately model complex interactions among concepts, as it limits all concept interactions to mere addition of their individual pixel values.
More recently, \citet{brady2024interaction} leverage interaction asymmetry properties to partially overcome this limitation. However, their approach still characterizes concept interactions using a restrictive parametric form (polynomials), which may not capture the full range of complex interactions in real-world data.
Thus, it remains a significant challenge to identify natural properties in the data-generating process that can support general concept interactions. 

\paragraph{Causal modularity and minimal changes.}
Humans understand and envision concept compositions through \emph{comparison} and \emph{analogy}~\citep{holyoak1989analogical,gentner1983structure}, cognitive processes that align with causal principles of \emph{modularity} and \emph{minimal changes}~\citep{spirtes2000causation,pearl2009causal,peters2017elements}.

\emph{Causal modularity} is the principle that a system's structure can be broken down into independent, autonomous modules, indicating how high-level concepts decompose into transferable low-level modules. The concept ``peacock'' breaks down into components (i.e., low-level concepts) like ``beak'', ``wings'', and ``colorful tail,'' while ``chicken'' decomposes into ``beak'', ``wings,'' and ``cockscomb'' (Figure~\ref{fig:teaser}). 
These components function as modular building blocks that can be recombined across contexts. The interaction patterns between these components are also transferable---the mechanism by which a ``beak'' interacts with ``rice'' forms a reusable module applicable across different bird species. This architecture enables efficient representation of complex concepts for humans.

The \emph{minimal-change} principle complements modularity by emphasizing that high-level concepts largely share low-level concepts, with only minimal distinguishing features. When comparing ``peacock'' and ``chicken'', both activate many identical low-level concepts (e.g., ``beaks'', ``wings''), differing primarily in specific attributes (``colorful tail'' vs. ``cockscomb''). This overlap facilitates comparison and analogy between related concepts. We intuitively recognize peacocks and chickens as more similar to each other than to fish precisely because they share more low-level concepts.

Together, these properties empower humans to envision novel concept combinations never directly experienced. 
Consider a novel combination ``peacock eating rice'' (Figure~\ref{fig:teaser}). We can readily imagine this because: (1) modularity allows decomposition of ``peacock'' into components including a ``beak'' and enables transfer of the interaction module ``beak \& rice'' observed in ``chicken eating rice'', and (2) the minimal-change principle enables the recognition that peacock beaks share properties with chicken beaks that interact with rice similarly, despite appearance differences (e.g., ``colorful tails''),

This explains our ability for compositional generalization -- we \emph{decompose high-level concepts into transferable low-level modules and leverage the shared features for analogical reasoning}, while accounting for minimal distinguishing features that preserve conceptual uniqueness.

\paragraph{Hierarchical data-generating processes.}
To encode these key properties, we formulate a hierarchical data-generating process to explicitly model concepts at distinct hierarchical levels and their interactions.
Let latent variables be $ \lat:= [ \lat_{1}, \cdots, \lat_{L} ]$, where $L$ denotes the total number of hierarchical levels and $\lat_{l} \in \R^{n(\lat_{l})}$ represents concept variables on the hierarchical level $l \in [L]$.
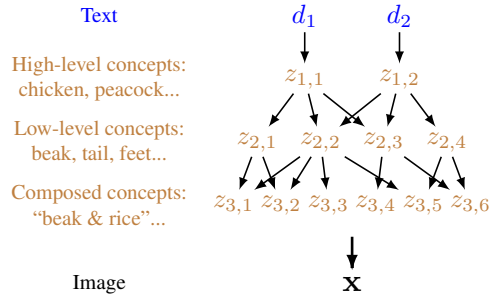
\begin{figure}[t]
    \centering
    \vspace{-0.4cm} 
    \begin{tikzpicture}[scale=.42, line width=0.6pt, inner sep=0.6mm, shorten >=.1pt, shorten <=.1pt]
      \tikzset{
        znode/.style={text=brown},
        dnode/.style={text=blue},
        xnode/.style={text=black},
        every node/.style={align=center},
        edge from parent/.style={draw,->}
      }

      \node[dnode] (d1) at (-1.5,1) {$d_{1}$};
      \node[dnode] (d2) at (1.5,1) {$d_{2}$};
      
      \node[znode] (z11) at (-1.5,-1) {$z_{1, 1}$};
      \node[znode] (z12) at (1.5,-1) {$z_{1, 2}$};
      
      \node[znode] (z21) at (-3,-3) {$z_{2,1}$};
      \node[znode] (z22) at (-1,-3) {$z_{2,2}$};
      \node[znode] (z23) at (1,-3) {$z_{2,3}$};
      \node[znode] (z24) at (3,-3) {$z_{2,4}$}; 
    
      \node[znode] (z31) at (-3.75,-5) {$z_{3,1}$};
      \node[znode] (z32) at (-2.25,-5) {$z_{3,2}$};
      \node[znode] (z33) at (-0.75,-5) {$z_{3,3}$};
      \node[znode] (z34) at (0.75,-5) {$z_{3,4}$};
      \node[znode] (z35) at (2.25,-5) {$z_{3,5}$};
      \node[znode] (z36) at (3.75,-5) {$z_{3,6}$};
      
      \draw[-latex] (d1) -- (z11);
      \draw[-latex] (d2) -- (z12);

      \draw[-latex] (z11) -- (z21);
      \draw[-latex] (z11) -- (z22);
      \draw[-latex] (z11) -- (z23);

      \draw[-latex] (z12) -- (z23);
      \draw[-latex] (z12) -- (z24);
      \draw[-latex] (z12) -- (z22);

      \draw[-latex] (z21) -- (z31);
      \draw[-latex] (z22) -- (z31);
      
      \draw[-latex] (z21) -- (z32);
      \draw[-latex] (z22) -- (z32);
      
      \draw[-latex] (z22) -- (z33);
      \draw[-latex] (z23) -- (z34);
      \draw[-latex] (z22) -- (z35);
      \draw[-latex] (z24) -- (z35);
      \draw[-latex] (z23) -- (z36);
      \draw[-latex] (z24) -- (z36);

      \node[xnode, scale=1.2] (x) at (0,-7.5) {$\obs$};
      
      \draw[-latex, line width=1pt] (0,-6) -- (0,-7) -- (x);

      \node[dnode, align=center, font={\footnotesize}] at (-8,1) {Text};
      \node[znode, align=center, font={\footnotesize}] at (-8,-1) {High-level concepts:\\chicken, peacock...};
      \node[znode, align=center, font={\footnotesize}] at (-8,-3) {Low-level concepts:\\beak, tail, feet...};
      \node[znode, align=center, font={\footnotesize}] at (-8,-5) {Composed concepts:\\``beak \& rice''...};
      \node[xnode, align=center, font={\footnotesize}] at (-8,-7.5) {Image};
    \end{tikzpicture}
    \caption{
        \textbf{A hierarchical data-generating process.}
        We denote the textual description as $\dis$, continuous latent concepts as $\lat$, and the image as $\obs$, where text-image pairs $p(\dis, \obs)$ are observable.
    }
    \label{fig:causal_graph}
    \vspace{-0.4cm}
\end{figure}
We denote the hierarchical graphical model as $ \cG := ( \cV, \cE ) $ where $ \cV:= \dis \cup \lat \cup \obs $ denotes the variable set and $ \cE $ denotes the edge set. \footnote{
    We view multidimensional variables as \emph{sets} when appropriate (e.g., $\obs$ as $\{ \obse_{i} \}_{i \in [d(\obs)]}$). 
}
We present the data-generating process in Eq.\ref{eq:data_generating_process} and Fig.~\ref{fig:causal_graph}.
\begin{align} \label{eq:data_generating_process}
    \late_{1, i} \sim p ( \late_{1, i} | \disc_{i} ), \quad
    v := g_{v} ( \parents{v}, \epsilon_{ v } ),
\end{align}
where $ v \in \cV \setminus \left( \lat_{1} \cup \dis\right) $ represents all non-root variables and $ \epsilon_{v} $ denotes its independent exogenous information. 
The discrete variable $\dis$ (i.e., textual descriptions) directly specifies high-level concepts $\lat_{1}$ (e.g., their presence or specific categories).
For instance, $\disc_{1} = 0$ may indicate the absence of ``peacock'', while $ \disc_{1} = 1 $ and $ \disc_{1} = 2$ might signify two varieties of ``peacocks''. 
We assume that the conditional distribution $p( \late_{1, i} | \disc_{i} = 0)$ is degenerate (i.e., a constant $\late_{1, i} $) to indicate its absence and has identical $ \supp{p( \late_{1, i} | \disc_{i} )} $ when $ \disc_{i} $ takes on different nonzero values to represent different varieties of the same concept $\late_{1, i}$.
We denote model parameters as $ \bm\theta := \left( p(\lat_{1}, \dis), \{ g_{v} \}_{v \in \cV \setminus \left(\lat_{1}, \dis\right) } \right) $. 
For exposition, we refer to $\obs$, $\dis$ as $\lat_{L+1}$, $\lat_{0}$ respectively.

\section{Composition Conditions and Identifiability} \label{sec:theory}

We first demonstrate how compositional generalization can arise from the hierarchical data-generating process present in natural data (Section~\ref{subsec:composition_conditions}).
Then, we show that one can learn such data-generating processes from image-text data $p(\obs, \dis)$ under proper assumptions (Section~\ref{subsec:identification}).

\subsection{Compositional Generalization Conditions} \label{subsec:composition_conditions}

Although the variables $ \{ \lat_{l} \}_{l \in [0, L+1]} $ form a Markov chain over $\lat_{l}$, the first module $ p( \lat_{1} | \dis ) $ could give rise to distinct supports $ \supp{\lat_{1} | \dis} $ across various values of $\dis$ (i.e., missing concepts).
In Figure~\ref{fig:problem_motivation}, the training support $\Ds$ lacks the combination ``peacock eating rice'' $\dis = [1, 1]$, which affects $\supp{\lat_{1}}$ and propagates downstream through $ \supp{\lat_{l}} $ for $l\in[L+1]$.
As children of $\lat_{1}$, variables $\lat_{2}$ only take on values from a more restricted set.
Consequently, the matching between two models $ \bm\theta $ and $ \hat{\bm\theta} $ is only partially supported due to the incompleteness of $\Ds$.
Our goal is to find conditions on the data-generating process $p(\obs|\dis)$ that permit generalization. We find this constraint is on the \emph{causal modules of the generating process}. A complex interaction (e.g., "beak \& rice") is modeled as a transferable module $g_{\late}$ for some latent $\late$. Compositional generalization to ``peacock eating rice'' ($\dis$) is possible if and only if this $g_{\late}$ module has already been learned from another example (e.g., ``chicken eating rice'', $\tilde{\dis}$).

\begin{restatable}[Composition Generalization]{theorem}{compositiongeneral} \label{thm:composition_general} {\ }
We assume the data-generating process \eqref{eq:data_generating_process}.
The discrete concept combination $\dis$ is composable (i.e., $\dis \in \Dcomp$) if for each continuous latent variable $\late \in \lat$, its parents' distribution support $ \supp{\parents{\late} | \dis } $ is contained by $\supp{\parents{\late} | \sdis}$ for some combination $\sdis\in \Ds $ on the support, i.e., $ \supp{\parents{ \late } | \dis } \subseteq \supp{ \parents{ \late } | \sdis } $.
\end{restatable}
\paragraph{Hierarchical structures directly enable the modular transfer.}
The key insight from Theorem~\ref{thm:composition_general} is that generalization to a new combination $\dis$ is possible if, for every latent variable $\late$ in the hierarchy, its required inputs ($\parents{\late}$) have already been seen in some training example $\tilde{\dis}$. \emph{Crucially, $\tilde{\dis}$ can be different for each $\late$}. This is the formal basis for analogy: we can learn the ``beak \& rice'' module from ``chicken eating rice'' ($\tilde{\dis}_1$) and the ``colorful tail'' module from ``peacock'' ($\tilde{\dis}_2$), then compose them to generate $\dis$.
This is a less restrictive condition than prior work. 
For example, a simple tree model, where high-level concepts branch into unique, non-interacting low-level concepts, is \emph{a special case of our framework} that recovers the disjoint-influence assumption in \citet{brady2023provably}. \emph{Our general hierarchical model (Figure~\ref{fig:causal_graph}) allows for shared low-level modules (e.g., ``beak'' is a child of both ``peacock'' and ``chicken''), which is precisely what enables modular transfer}. Furthermore, unlike models assuming global parametric forms (e.g., polynomials~\citep{brady2024interaction,lachapelle2023additive}), our interactions $g_{\late}$ are non-parametric modules learned from data.

In addition, the minimal change principle indicates that \emph{high-level concepts could share many low-level modules and differ only in a small set of concepts}.
In Figure~\ref{fig:teaser}, ``chicken'' and ``peacock'' share ``beak'' and ``wings'' and differ in a few concepts like ``colorful tail''.
Consequently, only a small fraction of modules need to transfer, making the composition more plausible.

\paragraph{Composability and sparsity.}
Theorem~\ref{thm:composition_general} highlights the crucial role of the graphical model's \textit{sparsity} for compositional generalization: a sparse model features smaller parental sets $ \parents{\late} $ which impose fewer constraints for the module transfer. 
This makes it more likely to find a $ \sdis \in \Ds $ on the support that includes its parents' support $ \supp{ \parents{\late} | \dis } $ for each variable $\late$. 
Thus, \emph{hierarchical models with sparser graphs offer stronger compositional capability.}
This is a key theoretical insight: compositional generalization is not just about what concepts are learned, but about the \textbf{sparsity of their causal interactions} in the latent hierarchy. This insight provides a clear theoretical motivation for encouraging sparsity during model learning, a principle we implement in Section~\ref{sec:method}.

\subsection{Causal Model Identification} \label{subsec:identification}


Section~\ref{subsec:composition_conditions} shows that a sparse, hierarchical causal model enables compositional generalization. However, this is only useful if such a latent model can be \emph{learned from data}. The challenge is \textbf{identifiability}: can we uniquely recover the true latent concepts $\lat$ and graph structure $\cG$ from only the observed text-image pairs $p(\lat, \dis)$? Without this, a model might learn an entangled, non-modular representation that fails to generalize.
We first define identifiability, which formalizes the equivalent class to which our learned representation recovers the true representation.

\begin{definition}[Component-wise Identifiability] \label{def:componentwise}
    Let $\lat \in \Zz$ and $\hat{\lat} \in \Zz$ be variables under two model specifications $ \bm\theta $ and $ \hat{\bm\theta} $ respectively. We say that $\lat$ and $\hat{\lat}$ are \emph{identified component-wise} if there exists a permutation $\pi$ such that for each $i \in [n(\lat)]$, $ \hat{\late}_{i} = h_{i}(\late_{\pi(i)})$ where $ h_{i} $ is invertible.
\end{definition}
Here, $\bm\theta$ represents the true model and $\hat{\bm\theta}$ represents the learned version.
Under the component-wise identifiability, our learned representation $\hat{z}_{i}$ captures complete information about a single variable $z_{\pi(i)}$ and no information from other variables $ z_{j} $ with $j\neq \pi(i)$.
This notion of identifiability is broadly adopted in the nonlinear independence component analysis literature (ICA)~\citep{hyvarinen2016unsupervised,hyvarinen2019nonlinear}.

In the following, we introduce and interpret conditions of the data-generating process that lead to component-wise identifiability over all the latent variables $\lat$. 

\paragraph{Remarks on the problem and our contribution.}
Identifying the latent \emph{hierarchical} models has long been a challenging task.
Much previous research has focused on hierarchical models with discrete variables~\citep{Pearl88,zhang2004hierarchical,choi2011learning,gu2023bayesian,kong2024learning} or assumes linear relations among variables~\citep{xie2022identification,huang2022latent,dong2023versatile,anandkumar2013learning}.
Unfortunately, both linearity and discreteness could be too restrictive to model complex real-world distributions of interest in this work (e.g., high-dimensional image distributions).
Closely related to our setting is prior work~\citet{kong2023identification} that admits nonlinear relations among the latent variables. 
They utilize conditional independence and sparse graphical conditions to provide identifiability guarantees for subspaces of latent variables, where latent dimensions can be entangled within certain groups.
While informative in many use cases, such subspace identifiability fails to reflect the granular graphical structure among individual concepts across levels.
For instance, multiple concepts at the same level (e.g., ``eyes'' and ``nose'') may be mixed into a single subspace, compromising the transferability of these individual concepts and limiting compositionality.
In contrast, we utilize the auxiliary information (e.g., the discrete concepts $\dis$) and assume that latent variables influence each other in a non-trivial manner, which we formalize in Condition~\ref{cond:identification}-\ref{asmp:linear_independence}. 
These conditions allow us to achieve component-wise identification (as opposed to the subspace identification~\citep{kong2023identification}) and fully identify the graphical structure, which is instrumental for compositional generalization.

\newcommand{\wvectorcontent}{%
&\mathbf{w}( \lat_{l+1}, \lat_{l} ) 
    = \Big(
        \frac{\partial \log p \left(\lat_{l+1} | \lat_{l} \right)}{\partial \late_{l+1, 1} }, \ldots, \frac{\partial \log p \left(\lat_{l+1} | \lat_{l} \right)}{\partial \late_{l+1, n( \lat_{l+1} ) } }, && \nonumber \\ 
        & \quad \frac{\partial^{2} \log p \left( \lat_{l+1} | \lat_{l} \right)}{(\partial \late_{l+1, 1} )^{2} }, \ldots, \frac{\partial^{2} \log p \left( \lat_{l+1} | \lat_{l} \right)}{\partial ( \late_{l+1, n( \lat_{l+1} ) } )^{2} }
    \Big). &&
}

\begin{restatable}[Identification Conditions]{condition}{identificationconditions} \label{cond:identification} {\ }
    \begin{enumerate}[label=\roman*,leftmargin=2em, topsep=0.5pt, partopsep=0pt, itemsep=-0.0em]
    \setlength\itemsep{-0.0em}
    
        \item \label{asmp:invertibility} [Invertibility]:
        There exists a smooth and invertible map $g_{l}: \left( \lat_{l}, \bm{\epsilon}_{l} \right) \mapsto \obs$ for $ l \in [0, L] $.

        \item \label{asmp:smooth_density} [Smooth Density]:
        The probability density function $ p( \lat_{l+1} | \lat_{l} ) $ is smooth.
    
        \item \label{asmp:conditional_independence} [Conditional Independence]: Components in $ \lat_{l+1} $ are independent given $ \lat_{l} $: $ p (\lat_{l+1} | \lat_{l}) = \prod_{n} p ( \late_{l+1, n} | \lat_{l} ) $.

        \item \label{asmp:linear_independence} [Sufficient Variability]: 
        For each value of $\lat_{l+1}$, there exist $2n(\lat_{l+1})+1$ values of $\lat_{l}$, i.e., $\lat_{l}^{(n)}$ with $n=0, 1, \dots, 2n(\lat_{l+1})+1$, such that the $2n( \lat_{l+1} )$ vectors $\ww(\lat_{l+1}, \lat_{l}^{(n)})-\ww(\lat_{l+1}, \lat_{l}^{(0)})$ are linearly independent, where vector $\ww(\lat_{l+1}, \lat_{l})$ is defined as follows:
       \begin{flalign}
            \wvectorcontent
        \end{flalign}
    \end{enumerate}
\end{restatable}

\paragraph{Discussion and interpretation.}
Condition~\ref{cond:identification}-\ref{asmp:invertibility} guarantees that the observed variables $\obs$ fully preserve the information in $\lat$, which is necessary since our goal is to recover $\lat$ from $\obs$.
Intuitively, the information accumulates from top to bottom in the hierarchical model and ultimately manifests as the observed variable $\obs$. This is plausible for many applications where $\obs$ (e.g., images) can be very high-dimensional and information-rich.
This condition is commonly employed in ICA literature~\citep{hyvarinen2016unsupervised,hyvarinen2019nonlinear,khemakhem2020ice,khemakhem2020variational,von2021self,kong2023identification}.
Condition~\ref{cond:identification}-\ref{asmp:smooth_density},\ref{asmp:conditional_independence} are also standard in the ICA literature.
In particular, Condition~\ref{cond:identification}-\ref{asmp:smooth_density} is a mild regularity condition on the conditional distributions, allowing us to measure distribution variations with density function derivatives.
Condition~\ref{cond:identification}-\ref{asmp:conditional_independence} assumes that the statistical dependence among latent variables on the same level originates from higher-level variables.
For instance, the dependence between a dog's ``eye'' and ``nose'' features stems from a higher-level concept like ``breed''.
Condition~\ref{cond:identification}-\ref{asmp:linear_independence} formalizes the intuition of ``sufficient variation'' among the latent variables.
In particular, the distributions of distinct low-level concepts $\late_{l+1, i}$, $\late_{l+1, j}$ \emph{vary differently} in response to their parent variables in $\lat_{l}$. 
For example, low-level concepts like ``eye'' and ``nose'' exhibit different patterns of change when the concept ``face'' varies, which enables humans to recognize them as separate concepts.
This condition is adopted and discussed extensively in prior work~\citep{hyvarinen2019nonlinear,khemakhem2020variational,kong2022partial}.

\begin{restatable}[Causal Module Identification]{theorem}{identification} \label{thm:identification} {\ }
    We assume the data-generating process \eqref{eq:data_generating_process}. 
    Under Condition~\ref{cond:identification}, we attain component-wise identifiability of $ \lat_{l} $ and the graphical structures $ \cG $ up to the index permutation at each level $l$.
\end{restatable}

\paragraph{Proof sketch.}
The crux is leveraging the influences from the high-level to the low-level latent variables in the hierarchical model.
Specifically, we utilize the textual description $\dis$ as the initial source of variation to identify the adjacent concepts $\lat_{1}$.
With $\lat_{1}$ identified, we can then use these variables to identify its children concepts $\lat_{2}$.
This process repeats through each level until we have fully identified all latent variables component-wise with permutation indeterminacy within each level.
Classic causal discovery algorithms (e.g., PC algorithm~\citep{spirtes2000causation}) can then process these identified latent variables to determine the graphical structure.

\paragraph{Implications.}
Theorem~\ref{thm:identification} provides a crucial link: it guarantees that the desirable data structure outlined in Theorem~\ref{thm:composition_general} is not just a theoretical construct, but is \emph{recoverable in practice from the observed distribution $p(\dis, \obs)$}. This two-step finding -- that hierarchy/sparsity enables generalization (Theorem~\ref{thm:composition_general}) and is learnable (Theorem~\ref{thm:identification}) under certain favorable conditions -- forms the principled foundation for our framework.


\begin{figure*}[t!]
    \centering
    \includegraphics[width=\linewidth]{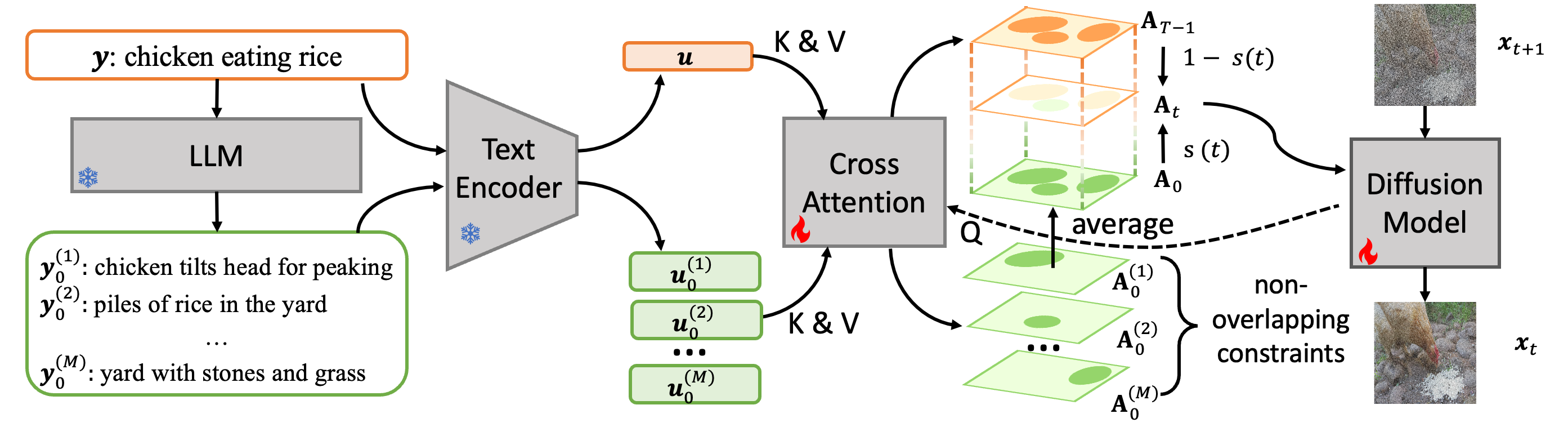}
    \caption{
        \ours.
        We first generate low-level text descriptions $\{ \yy_{0}^{(m)}\}_{m=1}^{M}$ from the original global description $\yy$ and obtain their embeddings $\uu$ and $ \{ \uu_{0}^{(m)}\}_{m=1}^{M} $.
        We average the $M$ low-level cross-attention outputs $ \{ \mA_{0}^{(m)}\}_{m=1}^{M} $ and interpolate it with the global cross-attention-map $\mA_{T-1}$ according to a step-dependent function $s(t)$.
        The resultant $\mA_{t}$ smoothly transitions from $ \mA_{T-1} $ to $ \mA_{0} $ in the generating process.
        We impose non-overlapping constraints to minimize unnecessary interactions among low-level concepts.
    }
    \label{fig:method_diagram}
    \vspace{-0.4cm}
\end{figure*}

\section{Empirical Validation with Diffusion Models} \label{sec:method}

In this section, we integrate key insights from Section~\ref{sec:theory} into existing diffusion models~\citep{rombach2022high} to enhance compositionality.

\paragraph{Hierarchical levels and diffusion steps.}
We conceptualize a diffusion model as a family of models $ \{ f_{t} \}_{t=1}^{T} $, where each $f_{t}$ restores $ \obs_{t+1} $ from its noisier version $ \obs_{t}$ by optimizing the variational evidence lower bound objective~\citep{sohl2015deep,ho2020denoising}:
\begin{align}
\begin{split}
\cL_{\mathrm{d}} &:= \sum_{t=1}^{T-1} \mathrm{KL} \left( q(\obs_{t} | \obs_{t+1}, \obs_{0}) \parallel p_{\!f_{t+1}} ( \obs_{t} | \obs_{t+1}, \yy) \right) \\
                    &\phantom{:=} - \log p_{\!f_{1}} ( \obs_{0} | \obs_{1}, \yy ),
\end{split}
\end{align}
where $q(\obs_{t} | \obs_{t+1}, \obs_{0})$ denotes the reverse diffusion process, $\mathrm{KL}$ stands for KL divergence, and $\yy$ refers to conditioning information (e.g., text).
As interpreted in prior work~\citep{kong2024learning}, $f_{t+1}$ extracts representation $ \lat_{\cS(t+1)} $ from the noisy data $ \obs_{t+1} $, and then employs $ \lat_{\cS(t+1)} $ and additional information $\yy$ to recover $ \obs_{t} $.
Here, $\lat_{\cS(t)}$ denotes latent variables with indices in a $t$-dependent set $\cS(t)$.
Higher noise levels (large $t$) corrupt low-level concepts in $\obs_{t}$, so the representation $ \lat_{\cS(t)} $ only retains high-level concepts.
Therefore, a higher noise level (a larger $t$) corresponds monotonically to a higher concept level $ \cS(t) $.
In Figure~\ref{fig:causal_graph}, if noise level $t$ just suffices to obscure low-level concepts $\lat_{2}$ (e.g., ``beak''), then $ \lat_{\cS(t)} $ corresponds to high-level concepts $\lat_{1}$ (e.g., ``peacock'').

\begin{figure*}[t!]
    \centering
    \includegraphics[width=1.0\linewidth]{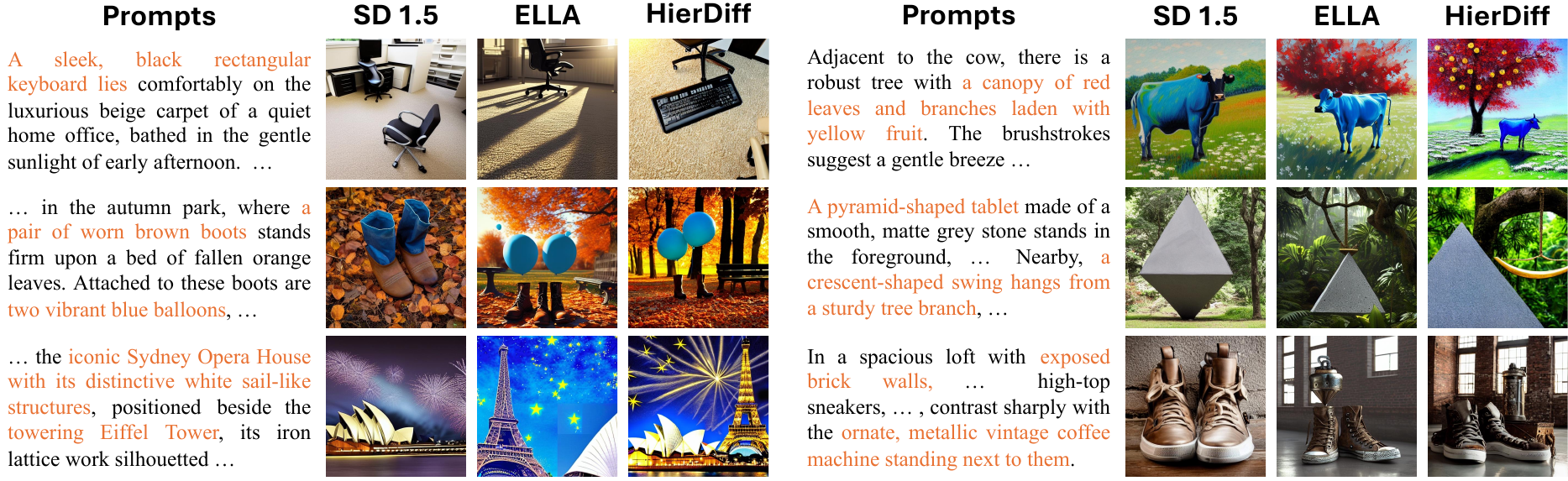}
    \caption{
        \textbf{Text-to-image generation results.}
        We highlight distinguishing tokens. 
    }
    \label{fig:main_text_to_image}
\end{figure*}

\paragraph{Hierarchical concept injection with sparsity control.}
Conventional approaches condition all generation steps with a global text prompt $\yy$~\citep{rombach2022high}.
However, our hierarchical-level interpretation suggests that only the information gap between $ \lat_{S(t+1)} $ and $ \lat_{S(t)} $ (i.e., the new details) is needed at step $t$.
Applying a global, invariant conditioning $\yy$ can limit the model's capacity, since it is compelled to disentangle and extract the desired information at each step. 
Moreover, this approach overlooks the naturally sparse structures in the hierarchical model, which may create unnecessary concept interactions and compromise composability as discussed in Section~\ref{sec:theory}.
Based on these insights, we formulate two key goals to improve existing methods:
\begin{enumerate}
    \item \textbf{Goal 1}: Inject step-specific information $\yy_{t}$;
    \item \textbf{Goal 2}: Encourage sparse interactions among concepts.
\end{enumerate}

\noindent
For \textbf{Goal 1}, we start with the high-level textual description $\yy$ (e.g., ``peacock eating rice''), which corresponds to a high-level concept $\dis$ in our theory. We produce detailed textual descriptions $ \yy_{0}:= \{ \yy^{(m)}_{0} \}_{m=1}^{M} $ for $M$ low-level concepts from the original high-level textual description $ \yy $.
This can be accomplished via a language model, as shown in prior work~\citet{feng2024layoutgpt,lian2023llm,wu2023harnessing,yang2024mastering} (see more in Appendix~\ref{app:lm_discussion}).
We treat these low-level descriptions as the information to be injected at the final step from $\obs_{1}$ to $\obs_{0}$ and the high-level description $\yy$ as information at the initial step from $\obs_{T}$ to $\obs_{T-1}$.
For the intermediate steps $ 0 < t < T-1 $, we interpolate the cross-attention outputs between the global text $ \yy $ and $\obs_{t}$, denoted as $ \mA_{T-1}:= \mathrm{XAttn} ( \obs_{t}, \uu) $ where $\uu$ denotes the text embedding of $\yy$, and cross-attention outputs between low-level descriptions $ \{ \yy_{0}^{(m)} \}_{m=1}^{M} $ and $\obs_{t}$, denoted as $ \mA_{0}^{(m)}:= \mathrm{XAttn} ( \obs_{t}, \uu_{0}^{(m)} ) $ where $\uu_{0}^{(m)}$ is the text embedding of $\yy_{0}^{(m)}$:
\begin{align}
    \begin{split}
    \mA_{t} := (1-s(t)) \cdot \mA_{T-1} + \frac{s(t)}{ M } \cdot \sum_{m=1}^{M} \mA_{0}^{(m)},
    \end{split}
\end{align}
where $s(t)$ is a monotonically decreasing function with $s(0)=1$ and $s(T-1)=0$.
We apply this modified cross-attention $ \mA_{t} $ at step $t$ for conditioning.
In this manner, we control the granularity of the injected information to match the diffusion step (i.e., high-level concepts at large steps).

\noindent
For \textbf{Goal 2}, we impose sparse regularization $\cL_{\mathrm{n}}$ on the overlaps among cross-attention maps $ \{\mH_{0}^{(m)}\}_{m=1}^{M} $ from low-level descriptions $ \{ \yy_{0}^{(m)} \}_{m=1}^{M} $ to minimize unnecessary spatial interactions among the $M$ low-level concepts:
\begin{align}
    \cL_{\mathrm{n}} := \sum_{m, n \in [M]: m \neq n} D \left( \mH_{0}^{(m)}, \mH_{0}^{(n)} \right),
\end{align}
where the DICE loss~\citep{sudre2017generalised,yeung2023calibrating} $ D(\mH_{1},\mH_{1}) := \frac{ 2 \cdot \text{tr} (\mH_{1}\mH_{2}) }{ \norm{\mH_{1}}_{1}+\norm{\mH_{2}}_{1} }$ measures the spatial overlap between attention maps $ \mH_{1} $ and $\mH_{2}$.
Under this regularization, concepts overlap sparsely with each other at each level, promoting sparse connectivity and thus composability.

\noindent
The overall training objective becomes:
\begin{align} \label{eq:training_objective}
    \cL := \cL_{\mathrm{d}} + \lambda \cdot \cL_{\mathrm{n}},
\end{align}
where $\lambda$ controls the regularization $\cL_{\mathrm{n}}$.
We refer to our method as \ours (Figure~\ref{fig:method_diagram}).

\begin{table}[t]
    \caption{
        \textbf{Evaluation results on DPG-Bench \citep{hu2024ella}}. 
        Baseline results are obtained from \citet{hu2024ella}.
        Our results are over three seeds.
    }
\label{tab:benchmark}
\vspace{-0.5cm}
\linespread{3.0}
\renewcommand\arraystretch{1.2}
\renewcommand\tabcolsep{2pt}
\begin{center}
\resizebox{\linewidth}{!}
{
   \begin{tabular}{l||c||cccccc}
\hline
Model & Score & Global & Entity & Attribute & Relation & Other \\
\hline
SD v2 \cite{rombach2022high} & 68.09 & 77.67 & 78.13 & 74.91 & 80.72 & 80.66 \\
PixArt-$\alpha$ \cite{chen2023pixart} & 71.11 & 74.97 & 79.32 & 78.60 & 82.57 & 76.96 \\
Playground v2 \cite{playground-v2}  & 74.54 & 83.61 & 79.91 & 82.67 & 80.62 & 81.22 \\
\hline
SD v1.5 \cite{rombach2022high} & 63.18 & 74.63 & 74.23 & 75.39 & 73.49 & 67.81 \\
ELLA \cite{hu2024ella} & 74.91 & 84.03 & 84.61 & 83.48 & 84.03 & 80.79 \\ \hline
\ours & \textbf{79.28}  & \textbf{85.77} & \textbf{85.15} & \textbf{86.98} & \textbf{86.82}  & \textbf{87.77} \\  \hline
\hline
\end{tabular}
}
\end{center}
\vspace{-0.5cm}
\end{table}

\paragraph{Theory \& practice.}
Our theoretical conditions and implementation are related as follows.

\noindent 1) \emph{Hierarchical processes} (Eq.~\ref{eq:data_generating_process} \& Cond.~\ref{cond:identification}-\ref{asmp:conditional_independence}):
The iterative diffusion chain $\obs_{T} \rightarrow ... \rightarrow \obs_0$ naturally models a hierarchical process. Our time-dependent conditioning injects concepts at the appropriate level (high-level concepts at high-noise steps, low-level at low-noise).

\noindent 2) \emph{Sparse connectivity} (Thm.~\ref{thm:composition_general}): Our theory identifies sparsity as critical for modular transfer. We enforce this via $\cL_\mathrm{n}$ which penalizes spatial overlap in attention maps, as a practical surrogate for encouraging a sparse latent graph.

\noindent 3) \emph{Level-dependent transformations} (Eq.~\ref{eq:data_generating_process}): Time-indexed diffusion models provide the required flexible, level-dependent transformations.

\noindent 4) \emph{Invertibility} (Cond.~\ref{cond:identification}-\ref{asmp:invertibility}): The diffusion model’s reconstruction objective (enforced by $\cL_{\mathrm{d}}$), which trains the model to denoise $\obs_{t}$ back to $\obs_0$, inherently promotes invertibility between exogenous noise, text, and images~\citep{kingma2014autoencoding,khemakhem2020variational}.

\noindent While other conditions (Condition~\ref{cond:identification}-\ref{asmp:smooth_density},\ref{asmp:linear_independence}) are assumptions on the data distribution itself.
While directly verifying the latent graph on real data is challenging, our framework provides a clear map from theory to implementation, our work provides a clear mapping from abstract theoretical principles to concrete implementation choices. The strong empirical performance of \ours, as we will show in Section~\ref{sec:exp}, validates the utility of our framework. Thus, this foundational understanding serves as a roadmap for future progress in compositional generalization.

\section{Experiments} \label{sec:exp}

\begin{table}[t]
    \caption{
            \textbf{Ablation studies on the DPG-Bench \cite{hu2024ella}}.
            TD and SR stand for time dependence and sparsity regularization.
        }

    \vspace{-0.2cm}
        \label{tab:ablation}
        \renewcommand\arraystretch{1.2}
\renewcommand\tabcolsep{2pt}
        \begin{tabular}{l||cc||c}
        \hline
		Metrics & -w/o TD & -w/o SR & \ours \\
		\hline
		Global & 85.15 $\pm$ 2.26 $\downarrow$ & 83.09 $\pm$ 2.49 $\downarrow$ & 85.77 $\pm$ 1.21 \\
		Entity & 84.95 $\pm$ 0.52 $\downarrow$ & 86.44 $\pm$ 0.58 & 85.15 $\pm$ 0.66 \\
		Attribute & 86.31 $\pm$ 0.24 $\downarrow$ & 86.56 $\pm$ 0.82 $\downarrow$ & 86.98 $\pm$ 0.09 \\
		Relation & 86.29 $\pm$ 1.03 $\downarrow$ & 87.14 $\pm$ 0.59 & 86.82 $\pm$ 0.86 \\
		Other & 85.64 $\pm$ 0.90 $\downarrow$ & 85.56 $\pm$ 0.60 $\downarrow$ & 87.77 $\pm$ 1.45 \\
		\hline
        \end{tabular}
\end{table}

\begin{figure}[t]
    \centering
        \includegraphics[width=\linewidth]{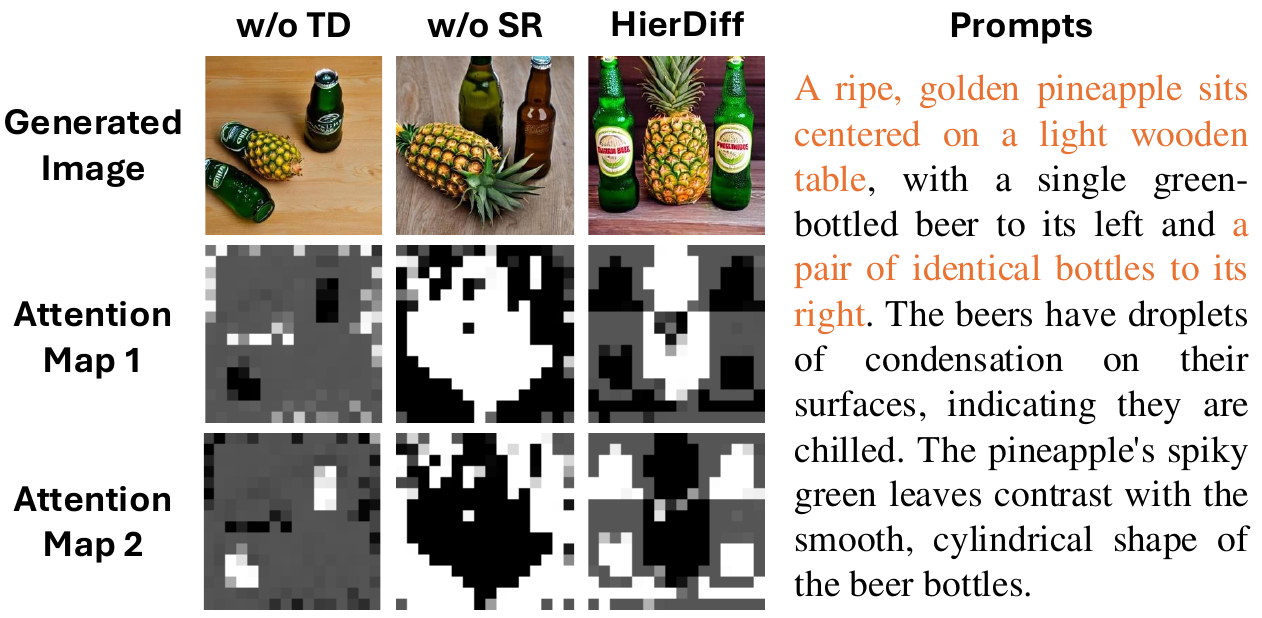}
        \caption{
            \textbf{Ablation studies.} 
        }
        \label{fig:main_ablation}
    \vspace{-0.4cm}
\end{figure}

\begin{figure*}[h!]
    \centering
    \begin{minipage}{0.66\linewidth}
        \centering
        \includegraphics[width=\linewidth]{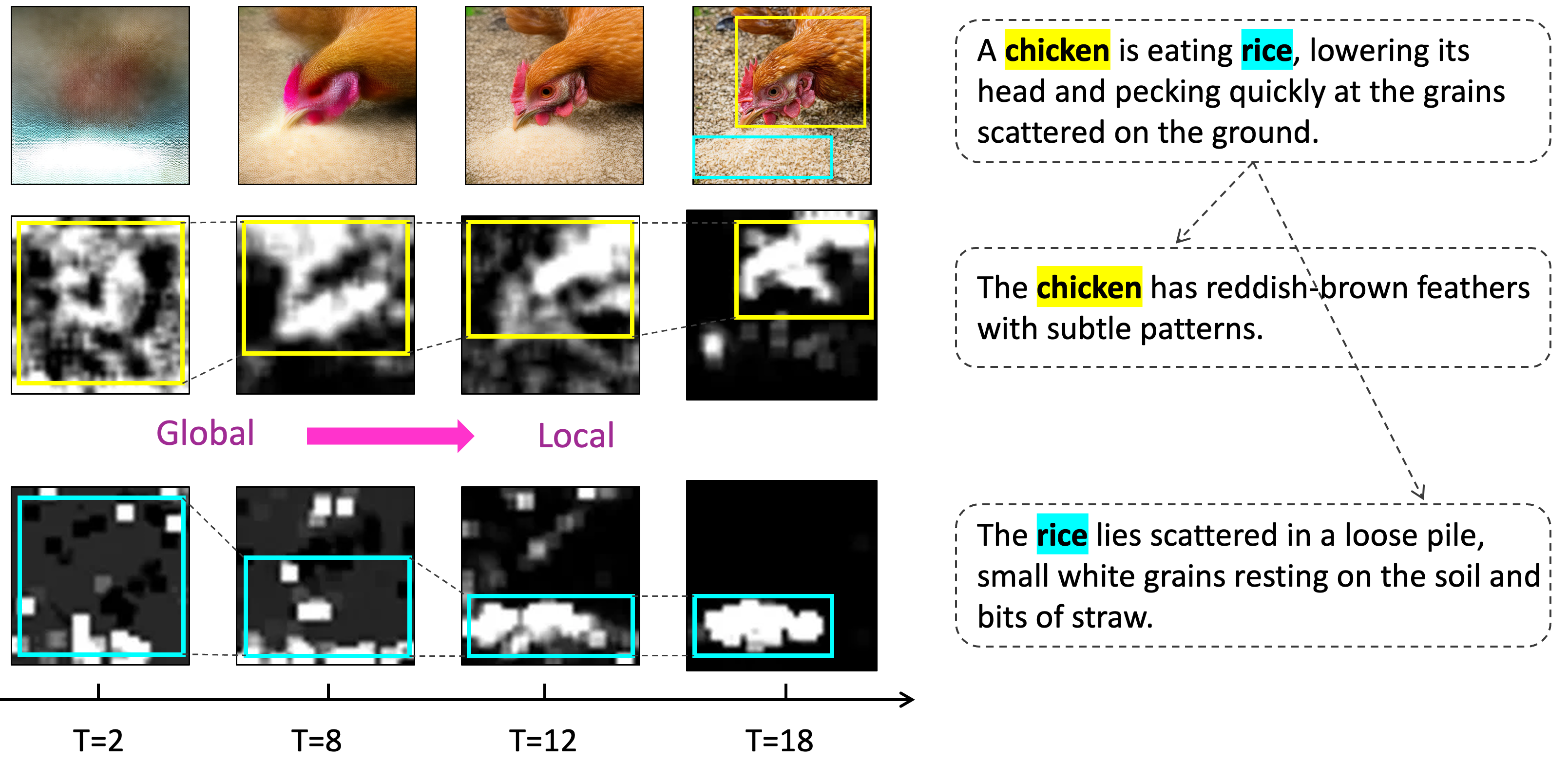}
        \caption{
            \textbf{Local cross-attention maps over steps.}
        }
        \label{fig:diffusion_step_attn}
    \end{minipage}
    \hfill
    \begin{minipage}{0.3\linewidth}
        \centering
        \captionof{table}{ \small
            \textbf{Comparison with SOTA on DPG-Bench~\citep{hu2024ella}}.
            Our approach applies to large-scale diffusion transformers beyond UNet-based models.
        }
        \label{tab:comp_sota}
        \begin{tabular}{l |c}
        \hline
        \textbf{Method} & \textbf{DPG $\uparrow$} \\
        \hline
        DALLE 3~\citep{dalle3_2023}      & 83.5 \\
        SD3-medium~\citep{esser2024scaling} & 84.1 \\
        FLUX-dev~\citep{flux2024}        & 84.0 \\
        FLUX-schnell~\citep{flux2024}    & 84.8 \\
        SANA-1.0~\citep{xie2024sana}     & 83.6 \\  
        SANA-1.5~\citep{xie2025sana}     & 84.7 \\   
        \hline
        \ours-DiT                       & \textbf{84.9} \\
        \hline
        \end{tabular}
    \end{minipage}
    \vspace{-0.3cm}
\end{figure*}

\paragraph{Setup.}
We fine-tune \ours from Stable Diffusion v1.5 \cite{rombach2022high}. 
Following \citet{hu2024ella}, we replace the CLIP text encoder with FLAN-T5-xl~\citep{raffel2020exploring} for enhanced text understanding capabilities, which we freeze during training.
For training, we use the public LayoutSAM dataset~\cite{zhang2024creatilayout}, which contains both the high-level text $\yy$ and corresponding low-level, local descriptions $\yy^{(m)}_{0}$.
At test time, given text $\yy$, we apply QWEN-v2.5~\cite{yang2024qwen2} to generate low-level local text descriptions similar to the training dataset (details in App.~\ref{app:lm_discussion} for an analysis of the LLM's reliability and robustness).
We adopt the interpolation function $s(t) = \cos \left( \frac{ \pi \cdot t }{ 2 (T-1) }\right)$, the number of local concepts $M=3$, and weight $\lambda=1e-4$ throughout our experiments.
We adopt DPG-Bench~\citep{hu2024ella}, which introduces five metrics, namely, Global, Entity, Attributes, Relation, and Other. 
The dataset comprises $1,065$ text prompts, each involving multiple objects/concepts with various relations.
More details are in Appendix~\ref{app:setup}.

\paragraph{Comparison with baselines.}
In Table~\ref{tab:benchmark},
\ours outperforms baseline methods across all evaluation metrics, demonstrating its superior capability in handling complex prompts involving multiple concepts and relationships.
Figure~\ref{fig:main_text_to_image} visualizes the results (more in Appendix~\ref{app:more_examples}).
For instance, in the case of the first prompt, only \ours successfully captures the ``keyboard'' concept and correctly renders its attributes (e.g., ``black'', ``sleek''), while baselines completely overlook this concept.

\paragraph{Ablation studies.}
We conduct ablation studies by sequentially removing the sparsity regularization $\cL_{\mathrm{n}}$ (w/o SR) and then the time dependence (w/o TD).
Table~\ref{tab:ablation} shows that both components contribute positively to the overall performance.
The `w/o TD' baseline validates our hierarchical approach against a simpler, non-hierarchical alternative that uses only the single global text prompt $\yy$ for all diffusion steps.
Notably, the time dependence significantly aids the model to understand complex relations among concepts (``Relation'' from $86.29$ to $87.14$), demonstrating the benefits of hierarchical structure to organize multiple concepts.
The sparsity regularization allows for precise control of individual concepts.
Figure~\ref{fig:main_ablation} visually demonstrates these findings.
Without sparsity constraints, the model's attention map lumps the two bottles together.
Comparing the attention maps from two local prompts, we observe that the sparsity constraint reduces the overlapping areas, enabling the model to control concepts separately.
Without the time dependence, the model fails to capture the relation between concepts, resulting in a confused mixture of ``pineapple'' and ``beer''.
See more examples in Appendix~\ref{app:more_ablation}.

\paragraph{Large-scale networks.}
To validate the scalability of our approach, we extend the implementations from the U-Net architecture to diffusion transformers with 4.8B parameters. As shown in Table~\ref{tab:comp_sota}, we can achieve comparable performance with \emph{billion-scale} models. Details in Appendix~\ref{app:setup}.

\paragraph{Visualization of composition.}
Figure~\ref{fig:diffusion_step_attn} visualizes the cross-attention maps from two low-level text descriptions (``chicken'' and ``rice'') across the diffusion steps with our model.
We can observe that under the two local-level text descriptions gradually shift from dispersed global attentions to more focused local attentions, and the intersection remains minimal. This verifies that our method can indeed facilitate composition by proper decomposition and re-composition with minimal interference.

\section{Conclusion and Limitations}
In this work, we connect compositional generalization with humans' cognitive process of drawing analogies.
We formalize this process via a hierarchical latent model that embodies causal modularity and minimal-change principles. 
Our framework accommodates complex concept interactions without restrictive assumptions. These theory insights lead to \ours, a T2I model that possesses competitive composition capabilities. 
\textbf{Limitations.}
Condition~\ref{cond:identification}-\ref{asmp:conditional_independence} assumes no direct causal influences among variables on each hierarchical level, which may restrict the representative power.
One may mitigate this with additional hierarchical levels to convert within-level to cross-level influences, or alternatively, consider more involved conditions~\citep{zhang2024causal}.

\clearpage

{
    \small
    \bibliographystyle{configurations/ieeenat_fullname}
    \bibliography{references}

@inproceedings{pearl1994probabilistic,
  title={A probabilistic calculus of actions},
  author={Pearl, Judea},
  booktitle={UAI 1994},
  pages={454--462},
  year={1994},
}

@article{pearl2009causal,
  title={Causal inference in statistics: An overview},
  author={Pearl, Judea and others},
  journal={Statistics Surveys},
  volume={3},
  pages={96--146},
  year={2009},
  publisher={The author, under a Creative Commons Attribution License}
}

@BOOK{Pearl88,
  AUTHOR =       "J. Pearl",
  TITLE =        "Probabilistic Reasoning in Intelligent Systems: Networks of Plausible Inference",
  PUBLISHER =    "Morgan Kaufmann",
  YEAR =         "1988",
}

@misc{hyvarinen2016unsupervised,
      title={Unsupervised Feature Extraction by Time-Contrastive Learning and Nonlinear ICA}, 
      author={Aapo Hyvarinen and Hiroshi Morioka},
      year={2016},
      eprint={1605.06336},
      archivePrefix={arXiv},
      primaryClass={stat.ML}
}

@inproceedings{hyvarinen2019nonlinear,
  title={Nonlinear ICA using auxiliary variables and generalized contrastive learning},
  author={Hyvarinen, Aapo and Sasaki, Hiroaki and Turner, Richard},
  booktitle={The 22nd International Conference on Artificial Intelligence and Statistics},
  pages={859--868},
  year={2019},
  organization={PMLR}
}

@inproceedings{khemakhem2020variational,
  title={Variational autoencoders and nonlinear ica: A unifying framework},
  author={Khemakhem, Ilyes and Kingma, Diederik and Monti, Ricardo and Hyvarinen, Aapo},
  booktitle={International Conference on Artificial Intelligence and Statistics},
  pages={2207--2217},
  year={2020},
  organization={PMLR}
}

@article{von2021self,
  title={Self-Supervised Learning with Data Augmentations Provably Isolates Content from Style},
  author={von K{\"u}gelgen, Julius and Sharma, Yash and Gresele, Luigi and Brendel, Wieland and Sch{\"o}lkopf, Bernhard and Besserve, Michel and Locatello, Francesco},
  journal={arXiv preprint arXiv:2106.04619},
  year={2021}
}

@misc{kingma2014autoencoding,
      title={Auto-Encoding Variational Bayes}, 
      author={Diederik P Kingma and Max Welling},
      year={2014},
      eprint={1312.6114},
      archivePrefix={arXiv},
      primaryClass={stat.ML}
}

@article{khemakhem2020ice,
  title={Ice-beem: Identifiable conditional energy-based deep models based on nonlinear ica},
  author={Khemakhem, Ilyes and Monti, Ricardo and Kingma, Diederik and Hyvarinen, Aapo},
  journal={Advances in Neural Information Processing Systems},
  volume={33},
  pages={12768--12778},
  year={2020}
}

@InProceedings{kong2022partial,
  title = 	 {Partial disentanglement for domain adaptation},
  author =       {Kong, Lingjing and Xie, Shaoan and Yao, Weiran and Zheng, Yujia and Chen, Guangyi and Stojanov, Petar and Akinwande, Victor and Zhang, Kun},
  booktitle = 	 {Proceedings of the 39th International Conference on Machine Learning},
  pages = 	 {11455--11472},
  year = 	 {2022},
  editor = 	 {Chaudhuri, Kamalika and Jegelka, Stefanie and Song, Le and Szepesvari, Csaba and Niu, Gang and Sabato, Sivan},
  volume = 	 {162},
  series = 	 {Proceedings of Machine Learning Research},
  month = 	 {17--23 Jul},
  publisher =    {PMLR},
  url = 	 {https://proceedings.mlr.press/v162/kong22a.html},
}

@article{choi2011learning,
  title={Learning latent tree graphical models},
  author={Choi, Myung Jin and Tan, Vincent YF and Anandkumar, Animashree and Willsky, Alan S},
  journal={Journal of Machine Learning Research},
  volume={12},
  pages={1771--1812},
  year={2011},
  publisher={Journal of Machine Learning Research}
}

@inproceedings{xie2022identification,
  title={Identification of linear non-gaussian latent hierarchical structure},
  author={Xie, Feng and Huang, Biwei and Chen, Zhengming and He, Yangbo and Geng, Zhi and Zhang, Kun},
  booktitle={International Conference on Machine Learning},
  pages={24370--24387},
  year={2022},
  organization={PMLR}
}

@article{huang2022latent,
  title={Latent Hierarchical Causal Structure Discovery with Rank Constraints},
  author={Huang, Biwei and Low, Charles Jia Han and Xie, Feng and Glymour, Clark and Zhang, Kun},
  journal={arXiv preprint arXiv:2210.01798},
  year={2022}
}

@inproceedings{anandkumar2013learning,
  title={Learning linear bayesian networks with latent variables},
  author={Anandkumar, Animashree and Hsu, Daniel and Javanmard, Adel and Kakade, Sham},
  booktitle={International Conference on Machine Learning},
  pages={249--257},
  year={2013}
}

@article{zhang2004hierarchical,
  title={Hierarchical latent class models for cluster analysis},
  author={Zhang, Nevin L},
  journal={The Journal of Machine Learning Research},
  volume={5},
  pages={697--723},
  year={2004},
  publisher={JMLR. org}
}

@book{spirtes2000causation,
	title={Causation, Prediction, and Search},
	author={Spirtes, Peter and Glymour, Clark N and Scheines, Richard},
	year={2000},
	publisher={MIT press}
}

@book{peters2017elements,
author = {Peters, J. and Janzing, D. and Sch\"olkopf, B.},
title = {Elements of Causal Inference: Foundations and Learning Algorithms},
address = {Cambridge, MA, USA},
publisher = {MIT Press},
year = {2017}
}

@misc{lachapelle2023additive,
      title={Additive Decoders for Latent Variables Identification and Cartesian-Product Extrapolation}, 
      author={Sébastien Lachapelle and Divyat Mahajan and Ioannis Mitliagkas and Simon Lacoste-Julien},
      year={2023},
      eprint={2307.02598},
      archivePrefix={arXiv},
      primaryClass={cs.LG}
}

@misc{du2024compositional,
      title={Compositional Generative Modeling: A Single Model is Not All You Need}, 
      author={Yilun Du and Leslie Kaelbling},
      year={2024},
      eprint={2402.01103},
      archivePrefix={arXiv},
      primaryClass={cs.LG}
}

@inproceedings{liu2022compositional,
  title={Compositional visual generation with composable diffusion models},
  author={Liu, Nan and Li, Shuang and Du, Yilun and Torralba, Antonio and Tenenbaum, Joshua B},
  booktitle={European Conference on Computer Vision},
  pages={423--439},
  year={2022},
  organization={Springer}
}

@inproceedings{rombach2022high,
  title={High-resolution image synthesis with latent diffusion models},
  author={Rombach, Robin and Blattmann, Andreas and Lorenz, Dominik and Esser, Patrick and Ommer, Bj{\"o}rn},
  booktitle={Proceedings of the IEEE/CVF conference on computer vision and pattern recognition},
  pages={10684--10695},
  year={2022}
}

@article{huang2023t2icompbench,
        title={T2I-CompBench: A Comprehensive Benchmark for Open-world Compositional Text-to-image Generation},
        author={Kaiyi Huang and Kaiyue Sun and Enze Xie and Zhenguo Li and Xihui Liu},
        journal={arXiv preprint arXiv: 2307.06350},
        year={2023}
    }

@article{ramesh2022hierarchical,
  title={Hierarchical text-conditional image generation with clip latents},
  author={Ramesh, Aditya and Dhariwal, Prafulla and Nichol, Alex and Chu, Casey and Chen, Mark},
  journal={arXiv preprint arXiv:2204.06125},
  volume={1},
  number={2},
  pages={3},
  year={2022}
}

@article{zhang2024realcompo,
  title={Realcompo: Dynamic equilibrium between realism and compositionality improves text-to-image diffusion models},
  author={Zhang, Xinchen and Yang, Ling and Cai, Yaqi and Yu, Zhaochen and Xie, Jiake and Tian, Ye and Xu, Minkai and Tang, Yong and Yang, Yujiu and Cui, Bin},
  journal={arXiv preprint arXiv:2402.12908},
  year={2024}
}

@article{hu2024ella,
  title={Ella: Equip diffusion models with llm for enhanced semantic alignment},
  author={Hu, Xiwei and Wang, Rui and Fang, Yixiao and Fu, Bin and Cheng, Pei and Yu, Gang},
  journal={arXiv preprint arXiv:2403.05135},
  year={2024}
}

@article{feng2023ranni,
  title={Ranni: Taming Text-to-Image Diffusion for Accurate Instruction Following},
  author={Feng, Yutong and Gong, Biao and Chen, Di and Shen, Yujun and Liu, Yu and Zhou, Jingren},
  journal={arXiv preprint arXiv:2311.17002},
  year={2023}
}

@article{feng2024layoutgpt,
  title={Layoutgpt: Compositional visual planning and generation with large language models},
  author={Feng, Weixi and Zhu, Wanrong and Fu, Tsu-jui and Jampani, Varun and Akula, Arjun and He, Xuehai and Basu, Sugato and Wang, Xin Eric and Wang, William Yang},
  journal={Advances in Neural Information Processing Systems},
  volume={36},
  year={2024}
}

@article{cho2023visual,
  title={Visual programming for text-to-image generation and evaluation},
  author={Cho, Jaemin and Zala, Abhay and Bansal, Mohit},
  journal={arXiv preprint arXiv:2305.15328},
  year={2023}
}

@article{yang2024mastering,
  title={Mastering text-to-image diffusion: Recaptioning, planning, and generating with multimodal llms},
  author={Yang, Ling and Yu, Zhaochen and Meng, Chenlin and Xu, Minkai and Ermon, Stefano and Cui, Bin},
  journal={arXiv preprint arXiv:2401.11708},
  year={2024}
}

@article{lian2023llm,
  title={Llm-grounded diffusion: Enhancing prompt understanding of text-to-image diffusion models with large language models},
  author={Lian, Long and Li, Boyi and Yala, Adam and Darrell, Trevor},
  journal={arXiv preprint arXiv:2305.13655},
  year={2023}
}

@inproceedings{bar2023multidiffusion,
  title={MultiDiffusion: Fusing Diffusion Paths for Controlled Image Generation},
  author={Bar-Tal, Omer and Yariv, Lior and Lipman, Yaron and Dekel, Tali},
  booktitle={International Conference on Machine Learning},
  pages={1737--1752},
  year={2023},
  organization={PMLR}
}

@article{chefer2023attend,
  title={Attend-and-excite: Attention-based semantic guidance for text-to-image diffusion models},
  author={Chefer, Hila and Alaluf, Yuval and Vinker, Yael and Wolf, Lior and Cohen-Or, Daniel},
  journal={ACM Transactions on Graphics (TOG)},
  volume={42},
  number={4},
  pages={1--10},
  year={2023},
  publisher={ACM New York, NY, USA}
}

@article{chen2023pixart,
  title={Pixart-backslashalpha : Fast training of diffusion transformer for photorealistic text-to-image synthesis},
  author={Chen, Junsong and Yu, Jincheng and Ge, Chongjian and Yao, Lewei and Xie, Enze and Wu, Yue and Wang, Zhongdao and Kwok, James and Luo, Ping and Lu, Huchuan and others},
  journal={arXiv preprint arXiv:2310.00426},
  year={2023}
}

@inproceedings{wu2023harnessing,
  title={Harnessing the spatial-temporal attention of diffusion models for high-fidelity text-to-image synthesis},
  author={Wu, Qiucheng and Liu, Yujian and Zhao, Handong and Bui, Trung and Lin, Zhe and Zhang, Yang and Chang, Shiyu},
  booktitle={Proceedings of the IEEE/CVF International Conference on Computer Vision},
  pages={7766--7776},
  year={2023}
}

@article{brady2024interaction,
  title={Interaction Asymmetry: A General Principle for Learning Composable Abstractions},
  author={Brady, Jack and von K{\"u}gelgen, Julius and Lachapelle, S{\'e}bastien and Buchholz, Simon and Kipf, Thomas and Brendel, Wieland},
  journal={arXiv preprint arXiv:2411.07784},
  year={2024}
}

@inproceedings{wiedemer2024provable,
title={Provable Compositional Generalization for Object-Centric Learning},
author={Thadd{\"a}us Wiedemer and Jack Brady and Alexander Panfilov and Attila Juhos and Matthias Bethge and Wieland Brendel},
booktitle={The Twelfth International Conference on Learning Representations},
year={2024},
url={https://openreview.net/forum?id=7VPTUWkiDQ}
}

@article{wiedemer2024compositional,
  title={Compositional generalization from first principles},
  author={Wiedemer, Thadd{\"a}us and Mayilvahanan, Prasanna and Bethge, Matthias and Brendel, Wieland},
  journal={Advances in Neural Information Processing Systems},
  volume={36},
  year={2024}
}

@article{kong2023identification,
  title={Identification of nonlinear latent hierarchical models},
  author={Kong, Lingjing and Huang, Biwei and Xie, Feng and Xing, Eric and Chi, Yuejie and Zhang, Kun},
  journal={Advances in Neural Information Processing Systems},
  volume={36},
  pages={2010--2032},
  year={2023}
}

@inproceedings{
kong2024learning,
title={Learning Discrete Concepts in Latent Hierarchical Models},
author={Lingjing Kong and Guangyi Chen and Biwei Huang and Eric P. Xing and Yuejie Chi and Kun Zhang},
booktitle={The Thirty-eighth Annual Conference on Neural Information Processing Systems},
year={2024},
url={https://openreview.net/forum?id=bO5bUxvH6m}
}

@article{gu2023bayesian,
  title={Bayesian pyramids: Identifiable multilayer discrete latent structure models for discrete data},
  author={Gu, Yuqi and Dunson, David B},
  journal={Journal of the Royal Statistical Society Series B: Statistical Methodology},
  volume={85},
  number={2},
  pages={399--426},
  year={2023},
  publisher={Oxford University Press US}
}

@inproceedings{
dong2023versatile,
title={A Versatile Causal Discovery Framework to Allow Causally-Related Hidden Variables},
author={Xinshuai Dong and Biwei Huang and Ignavier Ng and Xiangchen Song and Yujia Zheng and Songyao Jin and Roberto Legaspi and Peter Spirtes and Kun Zhang},
booktitle={The Twelfth International Conference on Learning Representations},
year={2024},
url={https://openreview.net/forum?id=FhQSGhBlqv}
}

@inproceedings{
zhang2024causal,
title={Causal Representation Learning from Multiple Distributions: A General Setting},
author={Kun Zhang and Shaoan Xie and Ignavier Ng and Yujia Zheng},
booktitle={Forty-first International Conference on Machine Learning},
year={2024},
url={https://openreview.net/forum?id=Pte6iiXvpf}
}

@inproceedings{brady2023provably,
  title={Provably learning object-centric representations},
  author={Brady, Jack and Zimmermann, Roland S and Sharma, Yash and Sch{\"o}lkopf, Bernhard and Von K{\"u}gelgen, Julius and Brendel, Wieland},
  booktitle={International Conference on Machine Learning},
  pages={3038--3062},
  year={2023},
  organization={PMLR}
}

@inproceedings{koh2021wilds,
  title={Wilds: A benchmark of in-the-wild distribution shifts},
  author={Koh, Pang Wei and Sagawa, Shiori and Marklund, Henrik and Xie, Sang Michael and Zhang, Marvin and Balsubramani, Akshay and Hu, Weihua and Yasunaga, Michihiro and Phillips, Richard Lanas and Gao, Irena and others},
  booktitle={International conference on machine learning},
  pages={5637--5664},
  year={2021},
  organization={PMLR}
}

@inproceedings{recht2019imagenet,
  title={Do imagenet classifiers generalize to imagenet?},
  author={Recht, Benjamin and Roelofs, Rebecca and Schmidt, Ludwig and Shankar, Vaishaal},
  booktitle={International conference on machine learning},
  pages={5389--5400},
  year={2019},
  organization={PMLR}
}

@article{taori2020measuring,
  title={Measuring robustness to natural distribution shifts in image classification},
  author={Taori, Rohan and Dave, Achal and Shankar, Vaishaal and Carlini, Nicholas and Recht, Benjamin and Schmidt, Ludwig},
  journal={Advances in Neural Information Processing Systems},
  volume={33},
  pages={18583--18599},
  year={2020}
}

@inproceedings{xie2023boxdiff,
  title={Boxdiff: Text-to-image synthesis with training-free box-constrained diffusion},
  author={Xie, Jinheng and Li, Yuexiang and Huang, Yawen and Liu, Haozhe and Zhang, Wentian and Zheng, Yefeng and Shou, Mike Zheng},
  booktitle={Proceedings of the IEEE/CVF International Conference on Computer Vision},
  pages={7452--7461},
  year={2023}
}

@inproceedings{feng2023trainingfree,
title={Training-Free Structured Diffusion Guidance for Compositional Text-to-Image Synthesis},
author={Weixi Feng and Xuehai He and Tsu-Jui Fu and Varun Jampani and Arjun Reddy Akula and Pradyumna Narayana and Sugato Basu and Xin Eric Wang and William Yang Wang},
booktitle={The Eleventh International Conference on Learning Representations },
year={2023},
url={https://openreview.net/forum?id=PUIqjT4rzq7}
}

@article{rassin2024linguistic,
  title={Linguistic binding in diffusion models: Enhancing attribute correspondence through attention map alignment},
  author={Rassin, Royi and Hirsch, Eran and Glickman, Daniel and Ravfogel, Shauli and Goldberg, Yoav and Chechik, Gal},
  journal={Advances in Neural Information Processing Systems},
  volume={36},
  year={2024}
}

@inproceedings{kim2023densediffusion,
  title={Dense Text-to-Image Generation with Attention Modulation},
  author={Kim, Yunji and Lee, Jiyoung and Kim, Jin-Hwa and Ha, Jung-Woo and Zhu, Jun-Yan},
  year={2023},
  booktitle = {ICCV}
}

@inproceedings{chen2024training,
  title={Training-free layout control with cross-attention guidance},
  author={Chen, Minghao and Laina, Iro and Vedaldi, Andrea},
  booktitle={Proceedings of the IEEE/CVF Winter Conference on Applications of Computer Vision},
  pages={5343--5353},
  year={2024}
}

@article{li2023divide,
  title={Divide \& bind your attention for improved generative semantic nursing},
  author={Li, Yumeng and Keuper, Margret and Zhang, Dan and Khoreva, Anna},
  journal={arXiv preprint arXiv:2307.10864},
  year={2023}
}

@article{huang2024t2i,
  title={T2i-compbench: A comprehensive benchmark for open-world compositional text-to-image generation},
  author={Huang, Kaiyi and Sun, Kaiyue and Xie, Enze and Li, Zhenguo and Liu, Xihui},
  journal={Advances in Neural Information Processing Systems},
  volume={36},
  year={2024}
}

@article{fang2023boosting,
  title={Boosting Text-to-Image Diffusion Models with Fine-Grained Semantic Rewards},
  author={Fang, Guian and Jiang, Zutao and Han, Jianhua and Lu, Guangsong and Xu, Hang and Liang, Xiaodan},
  journal={arXiv preprint arXiv:2305.19599},
  year={2023}
}

@article{sun2023dreamsync,
  title={Dreamsync: Aligning text-to-image generation with image understanding feedback},
  author={Sun, Jiao and Fu, Deqing and Hu, Yushi and Wang, Su and Rassin, Royi and Juan, Da-Cheng and Alon, Dana and Herrmann, Charles and van Steenkiste, Sjoerd and Krishna, Ranjay and others},
  journal={arXiv preprint arXiv:2311.17946},
  year={2023}
}

@article{xu2024imagereward,
  title={Imagereward: Learning and evaluating human preferences for text-to-image generation},
  author={Xu, Jiazheng and Liu, Xiao and Wu, Yuchen and Tong, Yuxuan and Li, Qinkai and Ding, Ming and Tang, Jie and Dong, Yuxiao},
  journal={Advances in Neural Information Processing Systems},
  volume={36},
  year={2024}
}

@article{wang2024divide,
  title={Divide and Conquer: Language Models can Plan and Self-Correct for Compositional Text-to-Image Generation},
  author={Wang, Zhenyu and Xie, Enze and Li, Aoxue and Wang, Zhongdao and Liu, Xihui and Li, Zhenguo},
  journal={arXiv preprint arXiv:2401.15688},
  year={2024}
}

@inproceedings{sudre2017generalised,
  title={Generalised dice overlap as a deep learning loss function for highly unbalanced segmentations},
  author={Sudre, Carole H and Li, Wenqi and Vercauteren, Tom and Ourselin, Sebastien and Jorge Cardoso, M},
  booktitle={Deep Learning in Medical Image Analysis and Multimodal Learning for Clinical Decision Support: Third International Workshop, DLMIA 2017, and 7th International Workshop, ML-CDS 2017, Held in Conjunction with MICCAI 2017, Qu{\'e}bec City, QC, Canada, September 14, Proceedings 3},
  pages={240--248},
  year={2017},
  organization={Springer}
}

@article{yeung2023calibrating,
  title={Calibrating the dice loss to handle neural network overconfidence for biomedical image segmentation},
  author={Yeung, Michael and Rundo, Leonardo and Nan, Yang and Sala, Evis and Sch{\"o}nlieb, Carola-Bibiane and Yang, Guang},
  journal={Journal of Digital Imaging},
  volume={36},
  number={2},
  pages={739--752},
  year={2023},
  publisher={Springer}
}

@article{zhang2024creatilayout,
    title={CreatiLayout: Siamese Multimodal Diffusion Transformer for Creative Layout-to-Image Generation},
    author={Zhang, Hui and Hong, Dexiang and Gao, Tingwei and Wang, Yitong and Shao, Jie and Wu, Xinglong and Wu, Zuxuan and Jiang, Yu-Gang},
    journal={arXiv preprint arXiv:2412.03859},
    year={2024}
  }

@article{ye2023ip,
  title={Ip-adapter: Text compatible image prompt adapter for text-to-image diffusion models},
  author={Ye, Hu and Zhang, Jun and Liu, Sibo and Han, Xiao and Yang, Wei},
  journal={arXiv preprint arXiv:2308.06721},
  year={2023}
}

@inproceedings{sohl2015deep,
  title={Deep unsupervised learning using nonequilibrium thermodynamics},
  author={Sohl-Dickstein, Jascha and Weiss, Eric and Maheswaranathan, Niru and Ganguli, Surya},
  booktitle={International conference on machine learning},
  pages={2256--2265},
  year={2015},
  organization={pmlr}
}

@article{ho2020denoising,
  title={Denoising diffusion probabilistic models},
  author={Ho, Jonathan and Jain, Ajay and Abbeel, Pieter},
  journal={Advances in neural information processing systems},
  volume={33},
  pages={6840--6851},
  year={2020}
}

@misc{playground-v2,
      url={[https://huggingface.co/playgroundai/playground-v2-1024px-aesthetic](https://huggingface.co/playgroundai/playground-v2-1024px-aesthetic)},
      title={Playground v2},
      year={2023},
      author={Li, Daiqing and Kamko, Aleks and Sabet, Ali and Akhgari, Ehsan and Xu, Linmiao and Doshi, Suhail}
}

@article{yang2024qwen2,
  title={Qwen2. 5 technical report},
  author={Yang, An and Yang, Baosong and Zhang, Beichen and Hui, Binyuan and Zheng, Bo and Yu, Bowen and Li, Chengyuan and Liu, Dayiheng and Huang, Fei and Wei, Haoran and others},
  journal={arXiv preprint arXiv:2412.15115},
  year={2024}
}

@article{raffel2020exploring,
  title={Exploring the limits of transfer learning with a unified text-to-text transformer},
  author={Raffel, Colin and Shazeer, Noam and Roberts, Adam and Lee, Katherine and Narang, Sharan and Matena, Michael and Zhou, Yanqi and Li, Wei and Liu, Peter J},
  journal={Journal of machine learning research},
  volume={21},
  number={140},
  pages={1--67},
  year={2020}
}

@article{xie2025sana,
  title={Sana 1.5: Efficient scaling of training-time and inference-time compute in linear diffusion transformer},
  author={Xie, Enze and Chen, Junsong and Zhao, Yuyang and Yu, Jincheng and Zhu, Ligeng and Wu, Chengyue and Lin, Yujun and Zhang, Zhekai and Li, Muyang and Chen, Junyu and others},
  journal={arXiv preprint arXiv:2501.18427},
  year={2025}
}

@inproceedings{esser2024scaling,
  title={Scaling rectified flow transformers for high-resolution image synthesis},
  author={Esser, Patrick and Kulal, Sumith and Blattmann, Andreas and Entezari, Rahim and M{\"u}ller, Jonas and Saini, Harry and Levi, Yam and Lorenz, Dominik and Sauer, Axel and Boesel, Frederic and others},
  booktitle={Forty-first international conference on machine learning},
  year={2024}
}

@misc{dalle3_2023,
  author       = {{OpenAI}},
  title        = {DALL·E 3},
  year         = {2023},
  howpublished = {\url{https://openai.com/dall-e-3}}
}

@misc{flux2024,
  author       = {{B. F. Labs}},
  title        = {Flux},
  year         = {2024},
  howpublished = {\url{https://github.com/black-forest-labs/flux}}
}

@article{xie2024sana,
  title={Sana: Efficient high-resolution image synthesis with linear diffusion transformers},
  author={Xie, Enze and Chen, Junsong and Chen, Junyu and Cai, Han and Tang, Haotian and Lin, Yujun and Zhang, Zhekai and Li, Muyang and Zhu, Ligeng and Lu, Yao and others},
  journal={arXiv preprint arXiv:2410.10629},
  year={2024}
}

@article{gentner1983structure,
  title={Structure-mapping: A theoretical framework for analogy},
  author={Gentner, Dedre},
  journal={Cognitive science},
  volume={7},
  number={2},
  pages={155--170},
  year={1983},
  publisher={Elsevier}
}

@article{holyoak1989analogical,
  title={Analogical mapping by constraint satisfaction},
  author={Holyoak, Keith J and Thagard, Paul},
  journal={Cognitive science},
  volume={13},
  number={3},
  pages={295--355},
  year={1989},
  publisher={Wiley Online Library}
}

@article{okawa2023compositional,
  title={Compositional abilities emerge multiplicatively: Exploring diffusion models on a synthetic task},
  author={Okawa, Maya and Lubana, Ekdeep S and Dick, Robert and Tanaka, Hidenori},
  journal={Advances in Neural Information Processing Systems},
  volume={36},
  pages={50173--50195},
  year={2023}
}

@inproceedings{trager2023linear,
  title={Linear spaces of meanings: compositional structures in vision-language models},
  author={Trager, Matthew and Perera, Pramuditha and Zancato, Luca and Achille, Alessandro and Bhatia, Parminder and Soatto, Stefano},
  booktitle={Proceedings of the IEEE/CVF International Conference on Computer Vision},
  pages={15395--15404},
  year={2023}
}

@article{uselis2025does,
  title={Does Data Scaling Lead to Visual Compositional Generalization?},
  author={Uselis, Arnas and Dittadi, Andrea and Oh, Seong Joon},
  journal={arXiv preprint arXiv:2507.07102},
  year={2025}
}
}

\clearpage
\onecolumn
\appendix

\begin{center}
    \Large\textbf{Appendix}
\end{center}


\section{Related Work} \label{app:related_work}

\paragraph{Compositional generalization.}
Compositional generalization has garnered significant attention from the generative model community, especially for text-to-image generation.
One avenue explores fine-tuning text-to-image models by incorporating feedback from image understanding systems as a form of reward~\citep{huang2024t2i,xu2024imagereward,sun2023dreamsync,fang2023boosting}.
However, this strategy may be limited by the text comprehension capabilities of models like CLIP.
Another approach involves adjusting the models' cross-attention mechanisms to align with the spatial and semantic details specified in the prompts~\citep{liu2022compositional,bar2023multidiffusion,li2023divide,rassin2024linguistic,chefer2023attend,feng2023trainingfree,chen2024training,xie2023boxdiff,kim2023densediffusion}. 
This approach relies on the interpretability of the foundational models and often results in only broad, suboptimal control over the generated images.
By leveraging the planning and reasoning strengths of language models, researchers have also broken down complex prompts into multiple regional descriptions, providing more precise conditions to guide the image generation process~\citep{cho2023visual,feng2023ranni,wang2024divide,yang2024mastering,lian2023llm,feng2024layoutgpt}. 
This decomposition aids in creating images that more accurately reflect the detailed components of the prompts.
These methods operate at the inference time and do not fundamentally learn disentangled concepts. 
Recent work~\citep{hu2024ella,wu2023harnessing} utilizes diffusion timesteps to modify the text embedding for refined generation control.
Nevertheless, \citet{hu2024ella} do not consider the spatial relations among concepts and fully rely on the pre-trained diffusion model's capacity.
\citet{wu2023harnessing} introduce additional inference-time optimization overhead and depend on the CLIP score as the optimization objective, which limits the generation quality with CLIP's capacity.
Guided by our theoretical insights, our work imposes proper constraints and modifications on the cross-entropy to learn disentangled concepts and their relations.

Although empirical studies are abundant in the field~\citep{du2024compositional, liu2022compositional, zhang2024realcompo, hu2024ella, huang2023t2icompbench, bar2023multidiffusion, yang2024mastering}, theoretical understanding remains limited and often hinges on restrictive assumptions about concept interactions. 
While recent generative~\citep{okawa2023compositional} and discriminative~\citep{uselis2025does,trager2023linear} models show compositional structures can emerge with scale, our work seeks to formally characterize the underlying causal data structures that guarantee such generalization.
From the theoretical standpoint, \citet{brady2023provably, wiedemer2024provable} consider concepts that affect disjoint pixel regions, effectively eliminating interaction between them. \citet{lachapelle2023additive} models the influences of concepts on the pixel space as purely additive, an approach that \citet{brady2024interaction} extends to include second-order polynomial terms. Additionally, \citet{wiedemer2024compositional} assumes direct access to the function governing concept interactions. These theoretical works also tend to overlook the varying levels of abstraction among concepts and their relationships within the latent space.
In contrast, thanks to the hierarchical structure, our theory admits compositions of transformations across hierarchical levels, allowing for complex interaction among concepts at different hierarchical levels.
Our sparsity regularizer ($\cL_{\mathrm{n}}$ in Eq.~\ref{eq:training_objective}) is related to the ``interaction asymmetry'' regularizer of, which also penalizes attention overlap, though \citet{brady2024interaction} use a pixel-wise product loss for VAE latent slots rather than a set-based DICE loss for time-dependent, text-conditioning in diffusion.

\paragraph{Latent hierarchical model identification.}
Modeling complex real-world data requires capturing hierarchical structures among latent variables. Prior work has explored identification conditions for such hierarchies with continuous latent variables influencing each other linearly~\citep{xie2022identification,huang2022latent,dong2023versatile,anandkumar2013learning}. 
Other studies focus on fully discrete cases, limiting their applicability to continuous data like images~\citep{Pearl88, zhang2004hierarchical, choi2011learning, gu2023bayesian}. Moreover, latent tree models connect variables through a single undirected path \citep{Pearl88,zhang2004hierarchical,choi2011learning}, which may oversimplify complex relationships.
Closely related to ours, \citet{kong2023identification} address nonlinear, continuous latent hierarchical models.
However, their framework cannot identify latent variables component-wise, leaving room for concept entanglement.
In contrast, we provide component-wise identifiability for latent variables and the graphical structures, along with transparent conditions for the data-generating process.

\section{Proofs for Theoretical Results}

\subsection{Proof for Theorem~\ref{thm:composition_general}}

\compositiongeneral*

\begin{proof}
    By definition, the concept combination $\dis$ is composable (i.e., $ \dis \in \Dcomp $) when the two alternative model specifications $ \bm\theta $ and $\hat{\bm\theta}$ agree on this specific $\dis$, i.e., $ \hat{g}_{z} = g_{z} $ for any $\late \in \lat$ over its inputs' support $ \cS_{\late} (\dis): = \supp{ \parents{z} | \dis } \times \supp{ \bm\epsilon_{z} } $. 
    We note that each exogenous variable $ \epsilon_{\late} $ is independent of $ \parents{\late} $ and its distribution remains invariant to the discrete variable $\dis$.
    We denote this relation as $ \bm\theta|_{\dis} = \hat{\bm\theta}|_{\dis} $.

    To derive this relation, we first show that under the assumption of the hierarchical data-generating process \eqref{eq:data_generating_process}, the specific model $\bm\theta := \left( p(\lat_{1}, \dis), \{ g_{v} \}_{v \in \cV \setminus \left(\lat_{1}, \dis\right) } \right) $'s behavior on the discrete concept space $ \Dcomp $ is fully determined by its behavior on the support $\Ds$.
    That is, if two specifications $\bm\theta$ and $\hat{\bm\theta}$ follow the hierarchical model assumption \eqref{eq:data_generating_process} and their behavior match over the support $\Ds$, this agreement would extend to $ \Dcomp $: $\forall \tilde{\dis} \in \Ds, \bm\theta|_{\tilde{\dis}} = \hat{\bm\theta}|_{\tilde{\dis}} \implies \forall \dis \in \Dcomp, \bm\theta|_{ \dis } = \hat{\bm\theta}|_{\dis}$.

    To this end, we assess the elementary generating function $ \late := g_{\late} ( \parents{ \late }, \epsilon_{\late} ) $ for every $\late \in \lat $ present in the hierarchical model.
    Although latent variables $ \{ \lat_{l} \}_{l \in [L+1]} $ form a Markov chain, the first module $ p( \lat_{1} | \dis ) $ may yield distinct supports $ \supp{\lat_{1} | \dis} $ across various values of $\dis$ (e.g., $\disc = 0$ for absence of the concept).
    Consequently, the matching of two models $ \bm\theta $ and $ \hat{\bm\theta} $ is only partially supported and depends on the specific value of $\dis$.
    We characterize a potentially larger composable space $ \Dcomp $ given their matching over the training support $\Ds$. 
    Under the theorem condition, we have $\supp{\parents{ \late } | \dis} $ at the specific $\dis$ is fully contained by $\supp{ \parents{ \late } | \sdis (\dis) }$ at some $ \sdis (\dis) \in \Ds$ dependent on $\dis$, i.e., 
    \begin{align} \label{eq:support_containment}
        \supp{\parents{ \late } | \dis} \subseteq \supp{ \parents{ \late } | \sdis(\dis) }.  
    \end{align}
    As the two models $ g_{\late} $ and $ \hat{g}_{\late} $ match over the discrete support $ \Ds $, this equality relation in \eqref{eq:support_containment} implies that this equality extends to $\sdis (\dis)$:
    \begin{align} \label{eq:matching}
        g_{\late} = \hat{g}_{\late}, \forall (\parents{\late}, \bm\epsilon_{\late}) \in \cS_{\late} (\sdis(\dis)).
    \end{align}
  
    As the relation in \eqref{eq:matching} holds true for all modules of $\bm\theta$ and $\hat{\bm\theta}$, the equality extends to the entire hierarchical model, i.e., $ \bm\theta|_{\dis} = \hat{\bm\theta}_{\dis} $ for $ \dis \in \Dcomp $, which concludes our proof.

\end{proof}



    

\subsection{Proof for Theorem~\ref{thm:identification}}

\renewcommand{\wvectorcontent}{%
&\mathbf{w}( \lat_{l+1}, \lat_{l} ) 
    = \Big(
        \frac{\partial \log p \left(\lat_{l+1} | \lat_{l} \right)}{\partial \late_{l+1, 1} }, \ldots, \frac{\partial \log p \left(\lat_{l+1} | \lat_{l} \right)}{\partial \late_{l+1, n( \lat_{l+1} ) } }, \quad \frac{\partial^{2} \log p \left( \lat_{l+1} | \lat_{l} \right)}{(\partial \late_{l+1, 1} )^{2} }, \ldots, \frac{\partial^{2} \log p \left( \lat_{l+1} | \lat_{l} \right)}{\partial ( \late_{l+1, n( \lat_{l+1} ) } )^{2} }
    \Big).
}

\identificationconditions*
\identification*

\begin{proof}
    We introduce Lemma~\ref{lemma:single_level_identification} from \citet{kong2022partial}, which identifies a trivial hierarchical model with only one latent level, i.e., $L=1$.

    \begin{restatable}[Single-level Identification~\citep{kong2022partial}]{lemma}{identification} \label{lemma:single_level_identification} {\ }
        We assume the following data-generating process \eqref{eq:data_generating_process}:
        \begin{align}\label{eq:one_level_data_generating_process}
            \lat \sim p(\lat | \uu), \quad \bm\epsilon \sim p(\bm\epsilon), \quad \obs := g( \lat, \bm\epsilon ),
        \end{align}
        where $\bm\epsilon$ refers to the exogenous variable independent of $\lat$ and $g$ stands for the generating function.
        Under Condition~\ref{cond:identification} with $L=1$ and $\lat_{0} = \uu$, we attain component-wise identifiability of $ \lat_{1} $.
    \end{restatable}

    In the general hierarchical case, we view the observed discrete variable $\dis$ as the top-level variable $\uu$ in \eqref{eq:one_level_data_generating_process} as the starting point.
    Lemma~\ref{lemma:single_level_identification} implies the component-wise identifiability of $\lat_{1}$.
    We then iteratively apply Lemma~\ref{lemma:single_level_identification} to identify level $\lat_{l+1}$ sequentially from top to bottom \eqref{eq:data_generating_process} by viewing the previously identified level $\lat_{l}$ as the conditioning variable $\uu$ in \eqref{eq:one_level_data_generating_process}.
    This reasoning gives the component-wise identifiability results for the entire hierarchical model.

    Since all the latent variables $\{ \late_{i} \}_{i=1}^{n(\lat)}$, we can view them as the observed variables. 
    The identifiability of the graphical structure $\cG$ follows from classic causal discovery methods (i.e., PC algorithm~\citep{spirtes2000causation}).

\end{proof}

\section{Additional Details for Experiments}

\subsection{Setup Details} \label{app:setup}
We train the model with a batch size of $800$ and a learning rate of $5e-5$. 
To inject multiple text conditions, we replicate the key and value linear layers in cross-attention, inspired by IP-Adapter \cite{ye2023ip}.
During testing, we prompt QWEN2.5~\citep{yang2024qwen2} with the instruction ``given a prompt X, segment it into three non-overlap descriptions (i.e., any two descriptions are not describing the same object), rewrite each subcaption to avoid interactions across each subcaption.''
For the experiments in Table~\ref{tab:comp_sota}, we employ the LayoutSAM dataset~\citep{zhang2024creatilayout} and finetune SANA-1.5~\citep{xie2025sana} with a batch size of $576$ for $20000$ steps at a learning rate of $5e-5$. We choose $ \lambda $ from $\{ 0.1, 1\}$.
The performance of \ours are over three random seeds, with a std of $0.1$.
\paragraph{Complexity.} Our method introduces moderate computational overhead. During training, it requires a (pre-computable) LLM forward pass to generate low-level prompts and computes the $\cL_\mathrm{n}$ sparsity loss, a simple DICE calculation on attention maps. During inference, it requires one LLM pass and $M$ cross-attention computations (where $M=3$ in our experiments) instead of one at each step. This represents a manageable trade-off for the significant improvement in compositional control.

\subsection{Language Model Usage} \label{app:lm_discussion}

We follow established practices~\citet{feng2024layoutgpt,lian2023llm,wu2023harnessing,yang2024mastering} to instruct QWENv2.5~\citep{yang2024qwen2} with a fixed instruction.
For example, QWENv2.5 rewrites ``a peacock is eating ice cream while...'' into ``A peacock is in the act of eating'', ..., ``a serving of ice cream is being visibly diminished''.
In our evaluation, we've found that QWEN2.5 performs decently for most examples, and more advanced models (Gemini 2.5 Flash, Claude 4) are superior on rare, challenging examples involving dense interactions of multiple concepts (e.g., detailed description of multiple mutually overlapping clothing items on a person). 
To quantify the performance of QWEN2.5, we instruct Claude 4 to evaluate the presence of high-level concepts in captions processed by QWEN2.5 and observe a $96\%$ success rate over $100$ DPG evaluation prompts.
We believe that the advancement of language models could further improve the performance.

\subsection{Additional Samples for Figure~\ref{fig:main_text_to_image}} \label{app:more_examples}
Figure~\ref{fig:app_text_to_image} and Figure~\ref{fig:app_text_to_image_more} display generated examples from \ours and baselines, with full text prompts.

\subsection{Additional Samples for Figure~\ref{fig:main_ablation}} \label{app:more_ablation}
Figure~\ref{fig:app_ablation} displays more examples for the ablation experiments in Figure~\ref{fig:main_ablation}.

\begin{figure*}[t]
 \begin{minipage}[t]{0.52\textwidth}
 \centering
       \textbf{Prompt}
    \end{minipage}%
    \hfill
    \begin{minipage}[t]{0.45\textwidth}
    \centering
      SD1.5 \cite{rombach2022high}
      \hfill
      ELLA \cite{hu2024ella}
      \hfill
      \textbf{Ours}
    \end{minipage}
    \noindent\rule{\textwidth}{0.4pt}
    \begin{minipage}[t]{0.52\textwidth}
        \vspace{0pt}  
        \textcolor{orange}{A sleek, black rectangular keyboard lies} comfortably on the luxurious beige carpet of a quiet home office, bathed in the gentle sunlight of early afternoon. The keys of the keyboard show signs of frequent use, and it's positioned diagonally across the plush carpet, which is textured with subtle patterns. Nearby, a rolling office chair with a high back and adjustable armrests sits invitingly, hinting at a quick break taken by its usual occupant.
    \end{minipage}%
    \hfill
    \begin{minipage}[t]{0.45\textwidth}
        \vspace{0pt}  
        \raisebox{0pt}[\height][0pt]{\includegraphics[width=0.32\textwidth]{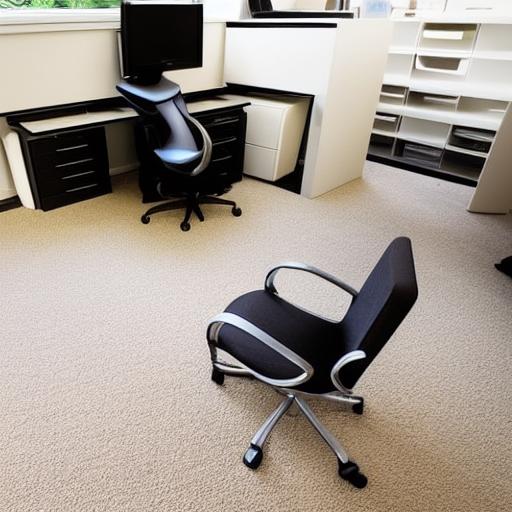}}%
        \hfill
        \raisebox{0pt}[\height][0pt]{\includegraphics[width=0.32\textwidth]{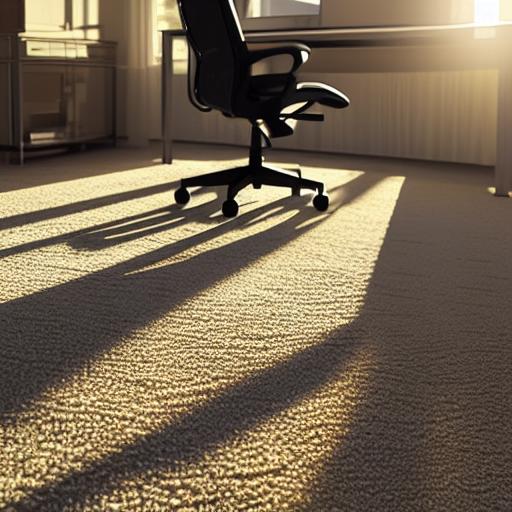}}%
        \hfill
        \raisebox{0pt}[\height][0pt]{\includegraphics[width=0.32\textwidth]{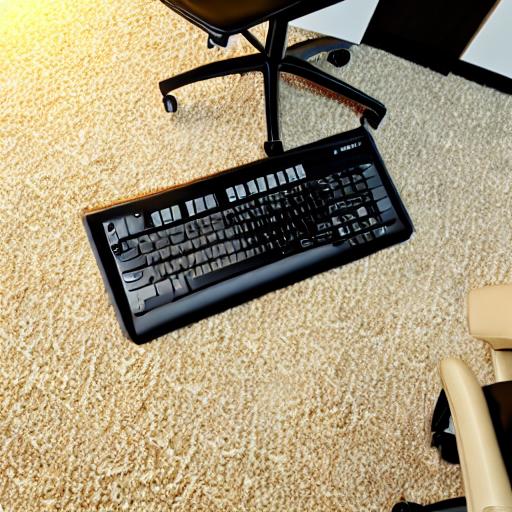}}
    \end{minipage}
    
    \noindent\rule{\textwidth}{0.4pt} 
    
    \begin{minipage}[t]{0.52\textwidth}
        \vspace{0pt}
        In the fading light of late afternoon, a scene unfolds in the autumn park, where \textcolor{orange}{a pair of worn brown boots} stands firm upon a bed of fallen orange leaves. Attached to these boots are \textcolor{orange}{two vibrant blue balloons}, gently swaying in the cool breeze. The balloons cast soft shadows on the ground, nestled among the trees with their leaves transitioning to auburn hues. Nearby, \textcolor{orange}{a wooden bench sits empty}, inviting passersby to witness the quiet juxtaposition of the still footwear and the dancing balloons.
    \end{minipage}%
    \hfill
    \begin{minipage}[t]{0.45\textwidth}
        \vspace{0pt}
        \raisebox{0pt}[\height][0pt]{\includegraphics[width=0.32\textwidth]{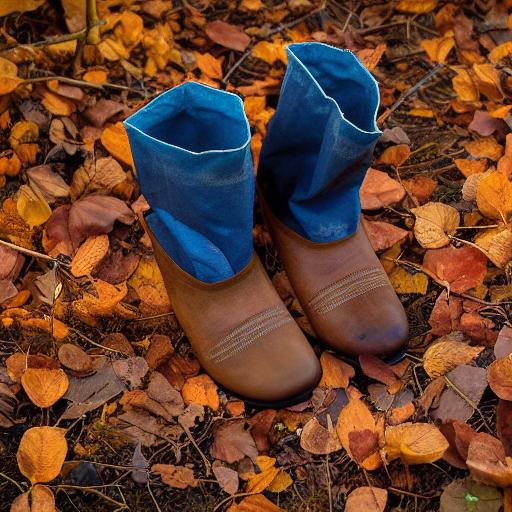}}%
        \hfill
        \raisebox{0pt}[\height][0pt]{\includegraphics[width=0.32\textwidth]{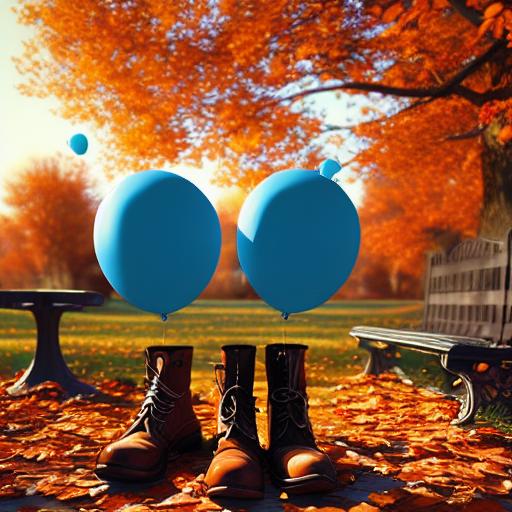}}%
        \hfill
        \raisebox{0pt}[\height][0pt]{\includegraphics[width=0.32\textwidth]{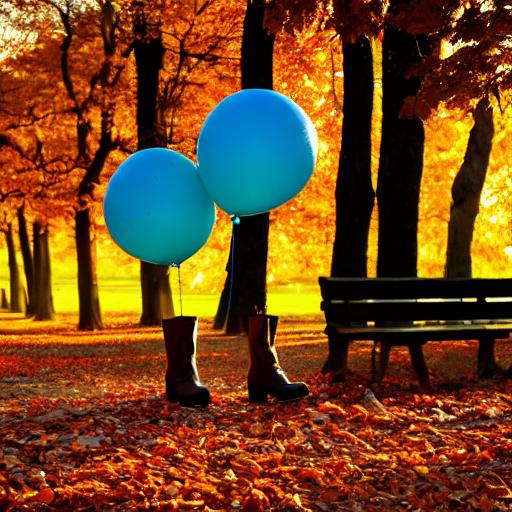}}
    \end{minipage}
    \noindent\rule{\textwidth}{0.4pt}
     \begin{minipage}[t]{0.52\textwidth}
        \vspace{0pt}
        A surreal composite image showcasing the \textcolor{orange}{iconic Sydney Opera House with its distinctive white sail-like structures}, positioned improbably beside the \textcolor{orange}{towering Eiffel Tower}, its iron lattice work silhouetted against the night. The backdrop is a vibrant blue sky, pulsating with dynamic energy, where yellow stars burst forth in a dazzling display, and swirls of deeper blue spiral outward. The scene is bathed in an ethereal light that highlights the contrasting textures of the smooth, shell-like tiles of the Opera House and the intricate metalwork of the Eiffel Tower.
    \end{minipage}%
    \hfill
    \begin{minipage}[t]{0.45\textwidth}
        \vspace{0pt}
        \raisebox{0pt}[\height][0pt]{\includegraphics[width=0.32\textwidth]{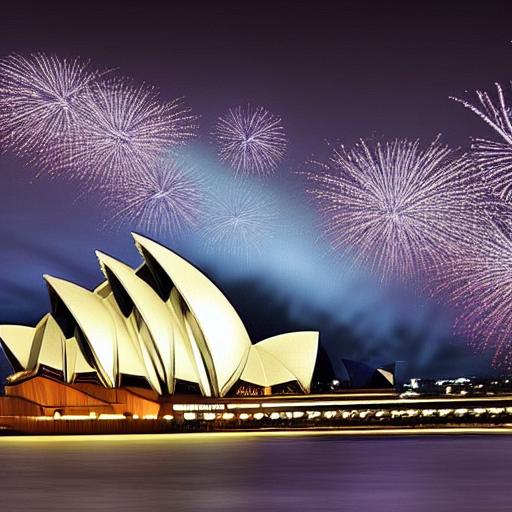}}%
        \hfill
        \raisebox{0pt}[\height][0pt]{\includegraphics[width=0.32\textwidth]{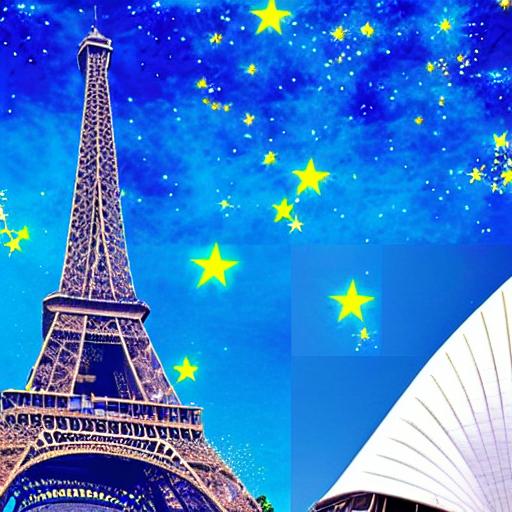}}%
        \hfill
        \raisebox{0pt}[\height][0pt]{\includegraphics[width=0.32\textwidth]{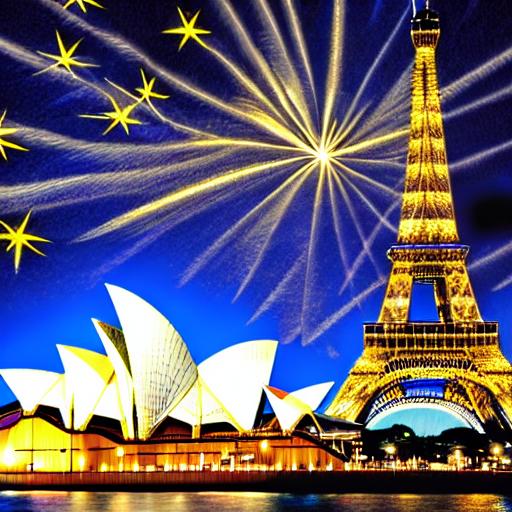}}
    \end{minipage}
     \noindent\rule{\textwidth}{0.4pt}
     \begin{minipage}[t]{0.52\textwidth}
        \vspace{0pt}
        An impressionistic painting depicts a vibrant blue cow standing serenely in a field of delicate white flowers. Adjacent to the cow, there is a robust tree with \textcolor{orange}{a canopy of red leaves and branches laden with yellow fruit}. The brushstrokes suggest a gentle breeze moving through the scene, and the cow's shadow is cast softly on the green grass beneath it.
    \end{minipage}%
    \hfill
    \begin{minipage}[t]{0.45\textwidth}
        \vspace{0pt}
        \raisebox{0pt}[\height][0pt]{\includegraphics[width=0.32\textwidth]{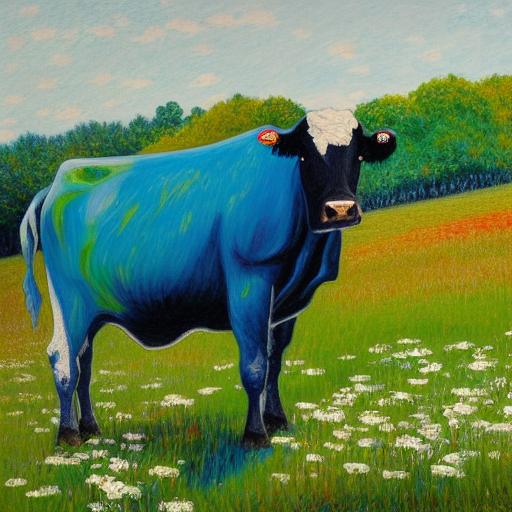}}%
        \hfill
        \raisebox{0pt}[\height][0pt]{\includegraphics[width=0.32\textwidth]{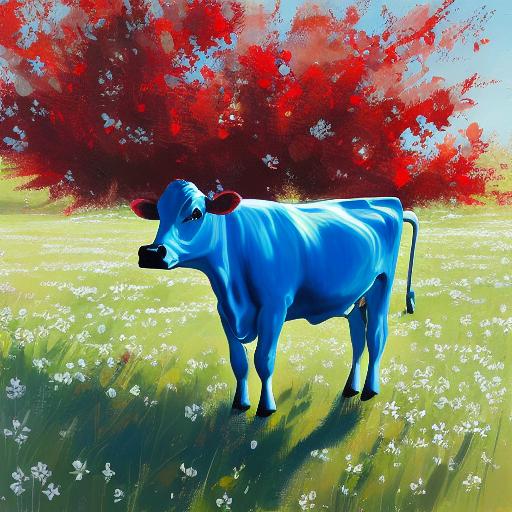}}%
        \hfill
        \raisebox{0pt}[\height][0pt]{\includegraphics[width=0.32\textwidth]{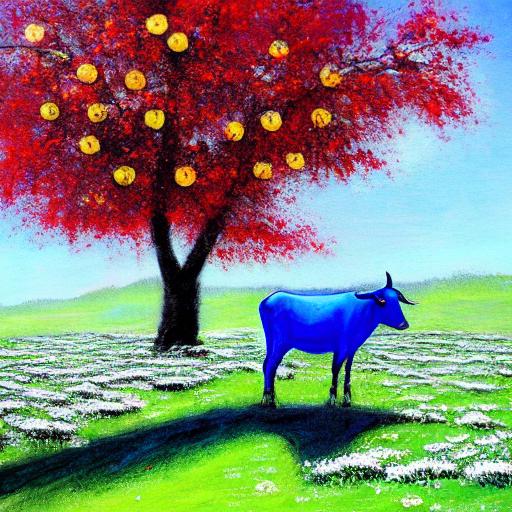}}
    \end{minipage}
     \noindent\rule{\textwidth}{0.4pt}
     \begin{minipage}[t]{0.52\textwidth}
        \vspace{0pt}
       \textcolor{orange}{a pyramid-shaped tablet} made of a smooth, matte grey stone stands in the foreground, its sharp edges contrasting with the wild, verdant foliage of the surrounding jungle. nearby, \textcolor{orange}{a crescent-shaped swing hangs from a sturdy tree branch}, crafted from a polished golden wood that glimmers slightly under the dappled sunlight filtering through the dense canopy above. the swing's smooth surface and gentle curve invite a sense of calm amidst the lush greenery.
    \end{minipage}%
    \hfill
    \begin{minipage}[t]{0.45\textwidth}
        \vspace{0pt}
        \raisebox{0pt}[\height][0pt]{\includegraphics[width=0.32\textwidth]{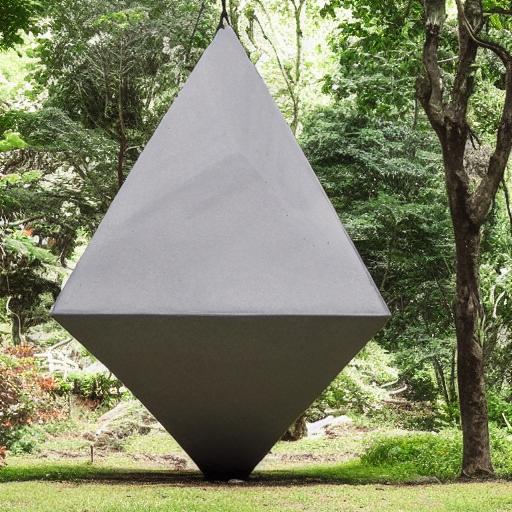}}%
        \hfill
        \raisebox{0pt}[\height][0pt]{\includegraphics[width=0.32\textwidth]{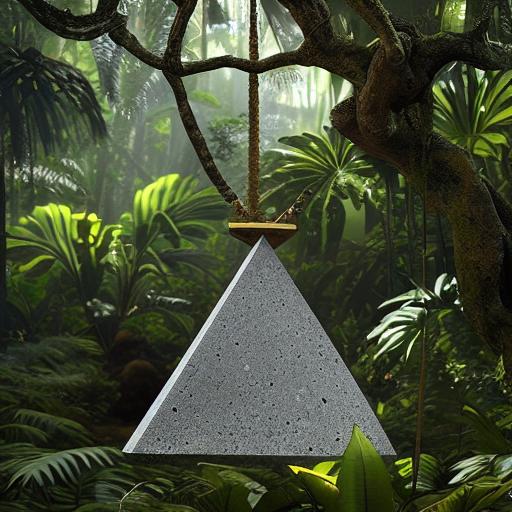}}%
        \hfill
        \raisebox{0pt}[\height][0pt]{\includegraphics[width=0.32\textwidth]{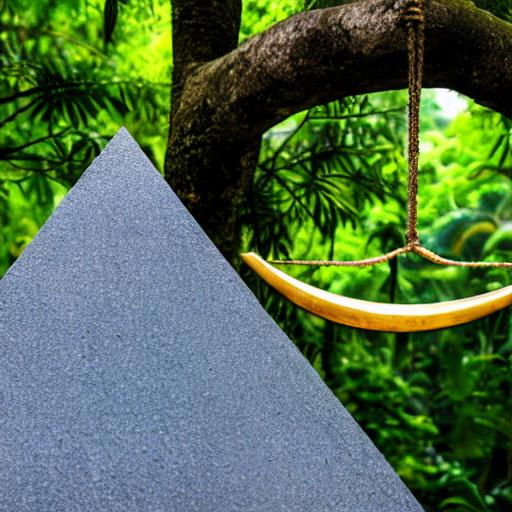}}
    \end{minipage}
     \noindent\rule{\textwidth}{0.4pt}
     \begin{minipage}[t]{0.52\textwidth}
        \vspace{0pt}
      In a spacious loft with high ceilings and \textcolor{orange}{exposed brick walls}, the morning light filters through large windows, casting a soft glow on a pair of trendy, high-top sneakers. These sneakers, made of rugged leather with bold laces, contrast sharply with the \textcolor{orange}{ornate, metallic vintage coffee machine standing next to them}. The coffee machine, with its intricate details and polished finish, reflects the light beautifully, setting a striking juxtaposition against the practical, street-style footwear on the polished concrete floor.
    \end{minipage}%
    \hfill
    \begin{minipage}[t]{0.45\textwidth}
        \vspace{0pt}
        \raisebox{0pt}[\height][0pt]{\includegraphics[width=0.32\textwidth]{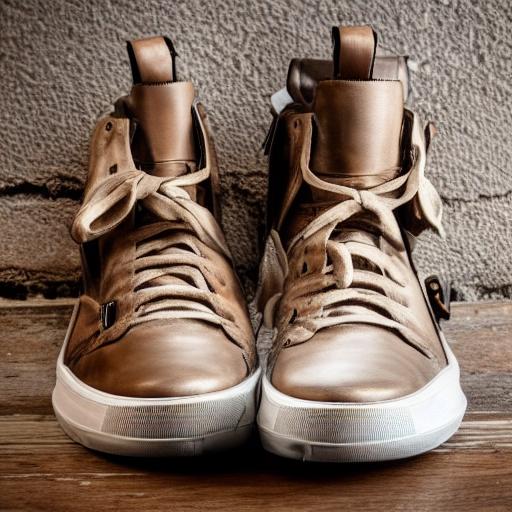}}%
        \hfill
        \raisebox{0pt}[\height][0pt]{\includegraphics[width=0.32\textwidth]{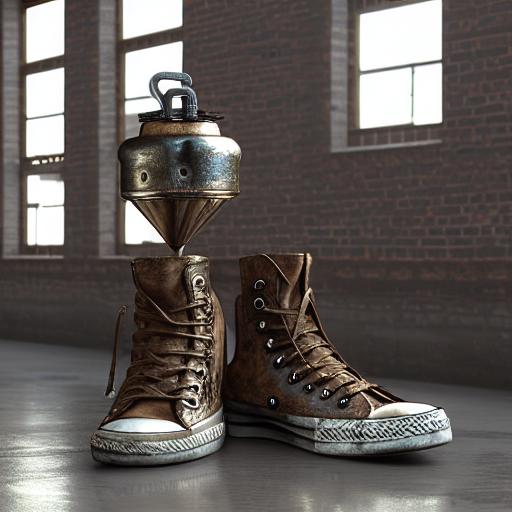}}%
        \hfill
        \raisebox{0pt}[\height][0pt]{\includegraphics[width=0.32\textwidth]{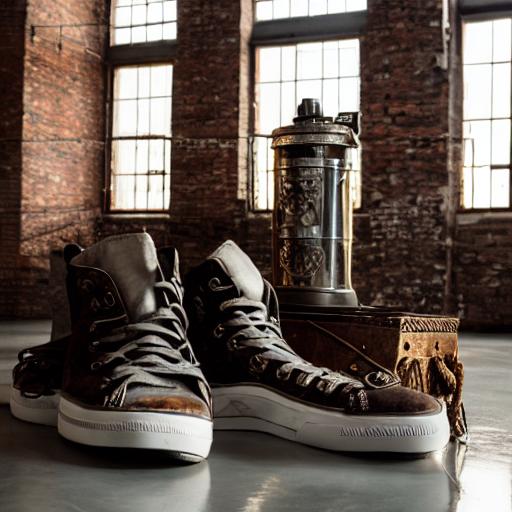}}
    \end{minipage}
     \noindent\rule{\textwidth}{0.4pt}
    \caption{
        \textbf{More text-to-image generation results.}
    }
    \label{fig:app_text_to_image}
\end{figure*}

\begin{figure*}[t]
 \begin{minipage}[t]{0.52\textwidth}
 \centering
       \textbf{Prompt}
    \end{minipage}%
    \hfill
    \begin{minipage}[t]{0.45\textwidth}
    \centering
      SD1.5 \cite{rombach2022high}
      \hfill
      ELLA \cite{hu2024ella}
      \hfill
      \textbf{Ours}
    \end{minipage}
    \noindent\rule{\textwidth}{0.4pt}
    \begin{minipage}[t]{0.52\textwidth}
        \vspace{0pt}  
       An elegant pair of glasses with a unique, \textcolor{orange}{gold hexagonal frame} laying on a smooth, dark wooden surface. The thin metal glints in the ambient light, highlighting the craftsmanship of the frame. The clear lenses reflect a faint image of the room's ceiling lights. To the side of the glasses, \textcolor{orange}{a leather-bound book} is partially open, its pages untouched.
    \end{minipage}%
    \hfill
    \begin{minipage}[t]{0.45\textwidth}
        \vspace{0pt}  
        \raisebox{0pt}[\height][0pt]{\includegraphics[width=0.32\textwidth]{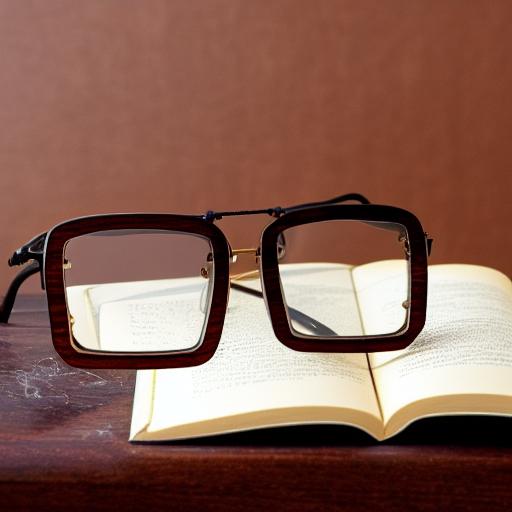}}%
        \hfill
        \raisebox{0pt}[\height][0pt]{\includegraphics[width=0.32\textwidth]{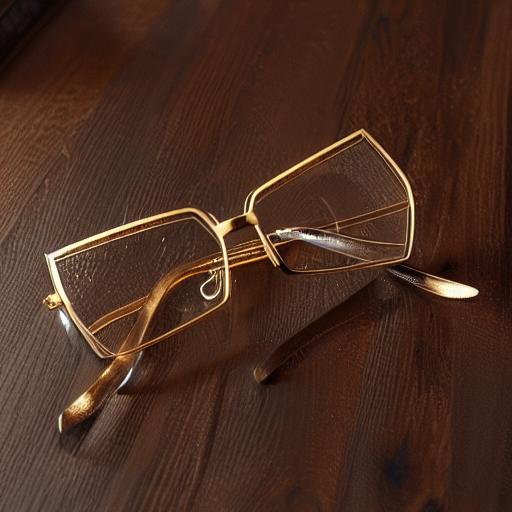}}%
        \hfill
        \raisebox{0pt}[\height][0pt]{\includegraphics[width=0.32\textwidth]{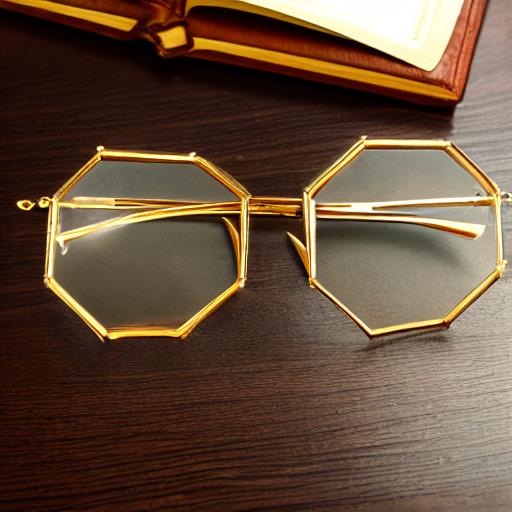}}
    \end{minipage}
    
    \noindent\rule{\textwidth}{0.4pt} 
    
    \begin{minipage}[t]{0.52\textwidth}
        \vspace{0pt}
      \textcolor{orange}{Two multicolored butterflies} with delicate, veined wings gently balance atop \textcolor{orange}{a vibrant, orange tangerine} in a bustling garden. The tangerine, with its glossy, dimpled texture, is situated on a wooden table, contrasting with the greenery of the surrounding foliage and flowers. The butterflies, appearing nearly small in comparison, add a touch of grace to the scene, complementing the natural colors of the verdant backdrop.
    \end{minipage}%
    \hfill
    \begin{minipage}[t]{0.45\textwidth}
        \vspace{0pt}
        \raisebox{0pt}[\height][0pt]{\includegraphics[width=0.32\textwidth]{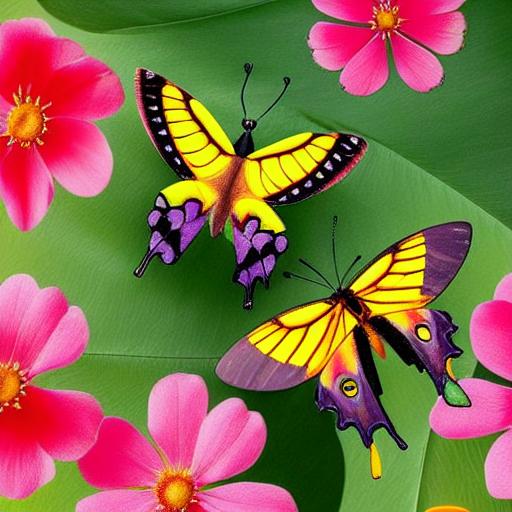}}%
        \hfill
        \raisebox{0pt}[\height][0pt]{\includegraphics[width=0.32\textwidth]{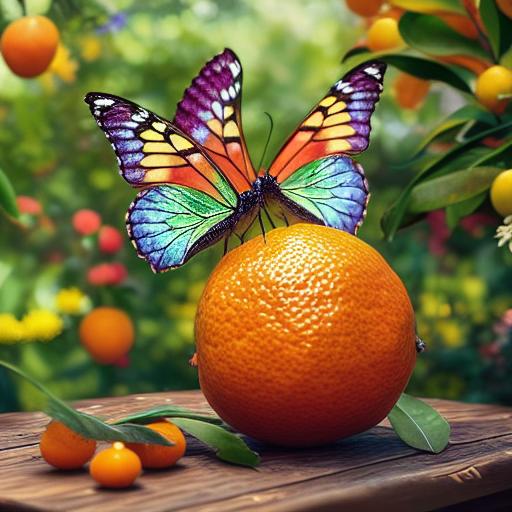}}%
        \hfill
        \raisebox{0pt}[\height][0pt]{\includegraphics[width=0.32\textwidth]{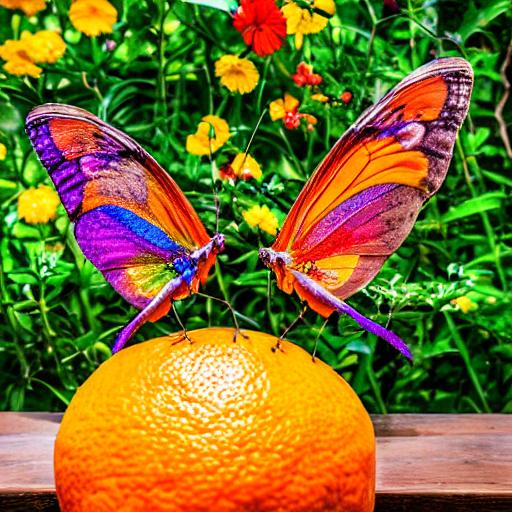}}
    \end{minipage}
    \noindent\rule{\textwidth}{0.4pt}
     \begin{minipage}[t]{0.52\textwidth}
        \vspace{0pt}
       \textcolor{orange}{Two sleek blue showerheads}, mounted against a backdrop of white ceramic tiles, release a steady stream of water. The water cascades down onto \textsc{orange}{a vivid, crisp green pear}that is centrally positioned directly beneath them. The pear's smooth and shiny surface gleams as the water droplets rhythmically bounce off, creating a tranquil, almost rhythmic sound in the otherwise silent bathroom.
    \end{minipage}%
    \hfill
    \begin{minipage}[t]{0.45\textwidth}
        \vspace{0pt}
        \raisebox{0pt}[\height][0pt]{\includegraphics[width=0.32\textwidth]{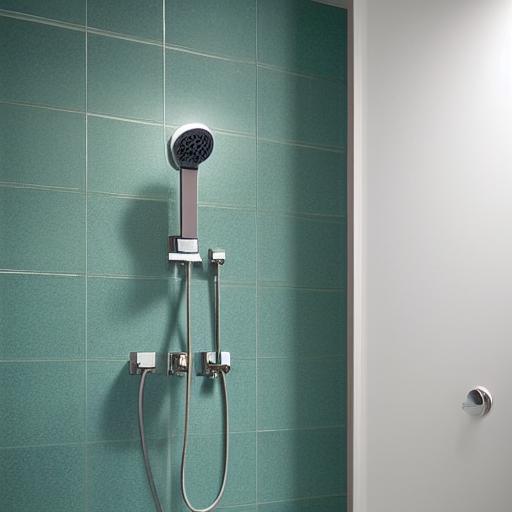}}%
        \hfill
        \raisebox{0pt}[\height][0pt]{\includegraphics[width=0.32\textwidth]{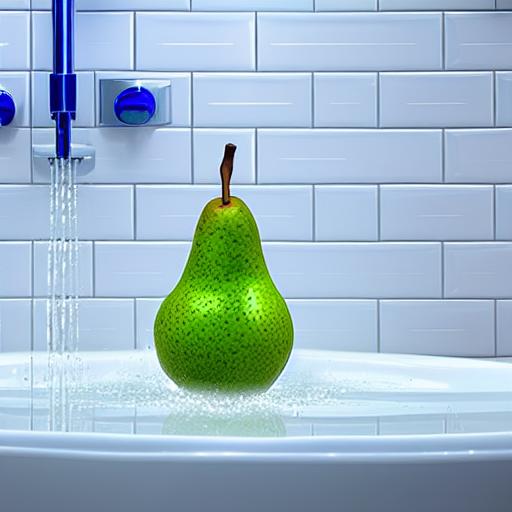}}%
        \hfill
        \raisebox{0pt}[\height][0pt]{\includegraphics[width=0.32\textwidth]{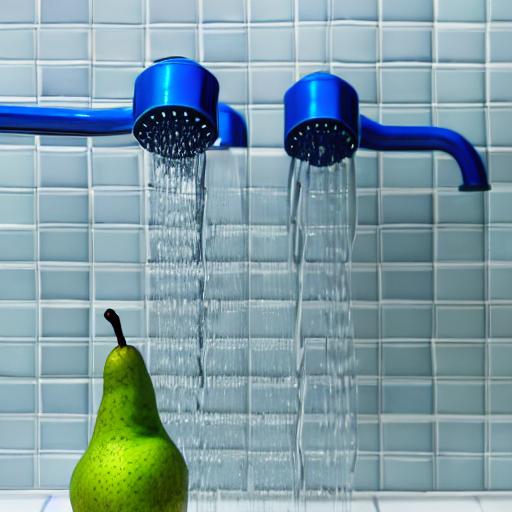}}
    \end{minipage}
     \noindent\rule{\textwidth}{0.4pt}
     \begin{minipage}[t]{0.52\textwidth}
        \vspace{0pt}
      In a modern kitchen, \textcolor{orange}{a square, chrome toaster} with a sleek finish sits prominently on the marble countertop, \textcolor{orange}{its size dwarfing the nearby red vintage rotary telephone}, which is placed quaintly on a wooden dining table. The telephone's vibrant red hue contrasts with the neutral tones of the kitchen, and its cord coils gracefully beside it. The polished surfaces of both the toaster and the telephone catch the ambient light, adding a subtle shine to their respective textures.
    \end{minipage}%
    \hfill
    \begin{minipage}[t]{0.45\textwidth}
        \vspace{0pt}
        \raisebox{0pt}[\height][0pt]{\includegraphics[width=0.32\textwidth]{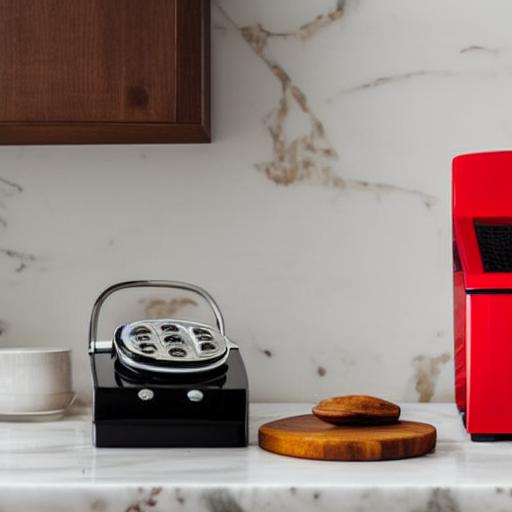}}%
        \hfill
        \raisebox{0pt}[\height][0pt]{\includegraphics[width=0.32\textwidth]{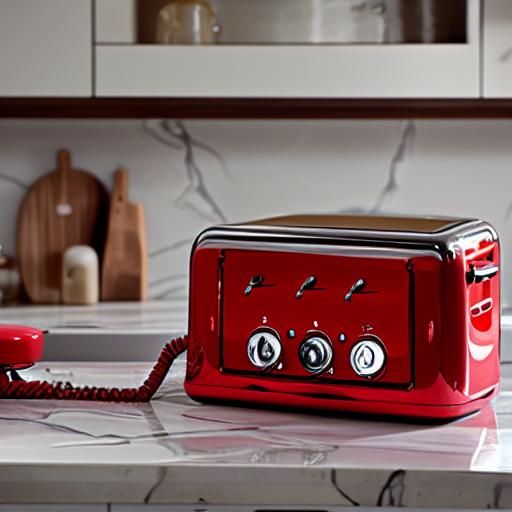}}%
        \hfill
        \raisebox{0pt}[\height][0pt]{\includegraphics[width=0.32\textwidth]{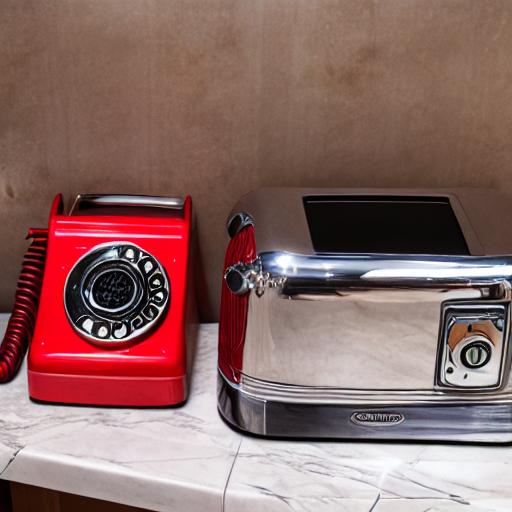}}
    \end{minipage}
     \noindent\rule{\textwidth}{0.4pt}
     \begin{minipage}[t]{0.52\textwidth}
        \vspace{0pt}
       \textcolor{orange}{Two slender bamboo-colored chopsticks} lie diagonally atop a smooth, round wooden cutting board with a rich grain pattern. The chopsticks, tapered to fine points, create a striking contrast against the cutting board's more robust and circular form. Around the board, \textcolor{orange}{there are flecks of freshly chopped green herbs and a small pile of julienned carrots}, adding a touch of color to the scene.
    \end{minipage}%
    \hfill
    \begin{minipage}[t]{0.45\textwidth}
        \vspace{0pt}
        \raisebox{0pt}[\height][0pt]{\includegraphics[width=0.32\textwidth]{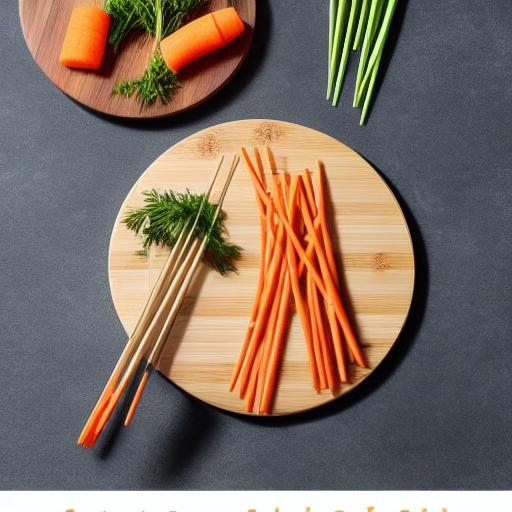}}%
        \hfill
        \raisebox{0pt}[\height][0pt]{\includegraphics[width=0.32\textwidth]{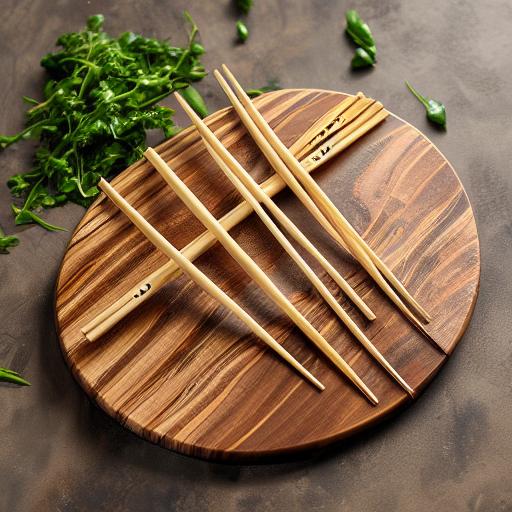}}%
        \hfill
        \raisebox{0pt}[\height][0pt]{\includegraphics[width=0.32\textwidth]{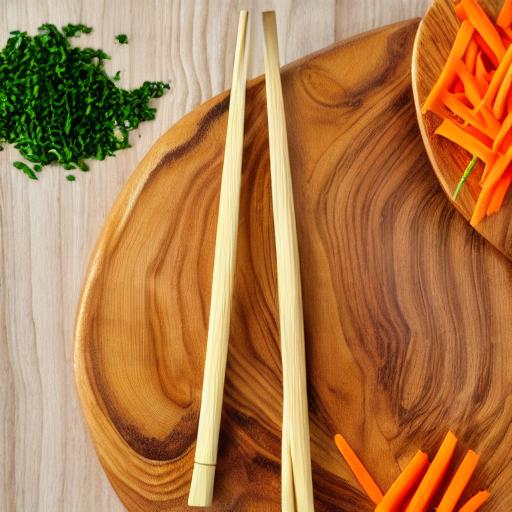}}
    \end{minipage}
     \noindent\rule{\textwidth}{0.4pt}
     \begin{minipage}[t]{0.52\textwidth}
        \vspace{0pt}
     A cozy bathroom features a pristine, white claw-foot bathtub on a backdrop of pastel green tiles. \textcolor{orange}{Adjacent to the tub, a tower of soft, white toilet paper is neatly stacked}, glimmering gently in the diffuse glow of the afternoon sunlight streaming through a frosted window. The gentle curvature of the tub contrasts with the straight lines of the stack, creating a harmonious balance of shapes within the intimate space.
    \end{minipage}%
    \hfill
    \begin{minipage}[t]{0.45\textwidth}
        \vspace{0pt}
        \raisebox{0pt}[\height][0pt]{\includegraphics[width=0.32\textwidth]{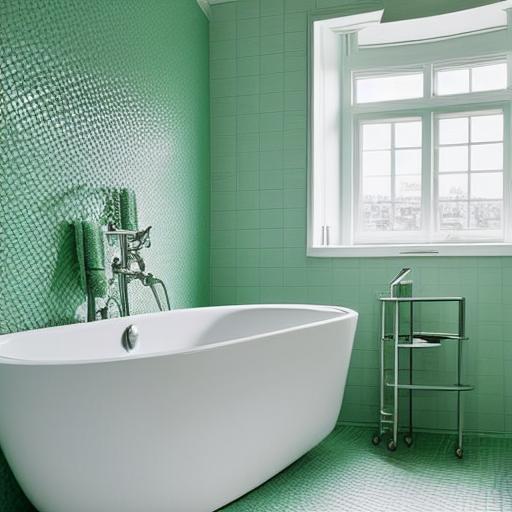}}%
        \hfill
        \raisebox{0pt}[\height][0pt]{\includegraphics[width=0.32\textwidth]{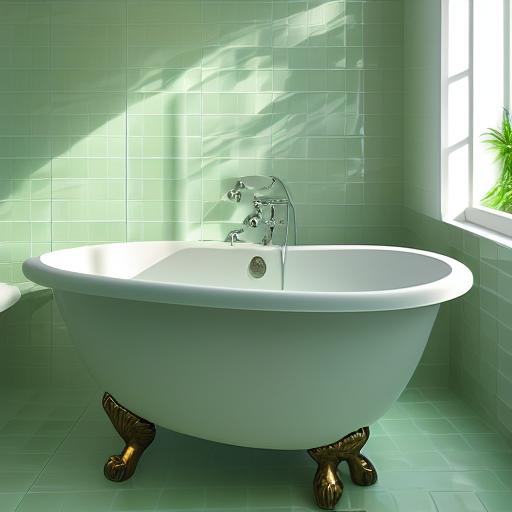}}%
        \hfill
        \raisebox{0pt}[\height][0pt]{\includegraphics[width=0.32\textwidth]{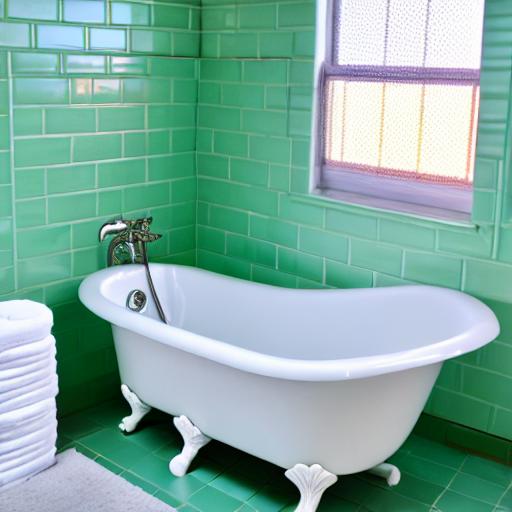}}
    \end{minipage}
     \noindent\rule{\textwidth}{0.4pt}
    \caption{
        \textbf{More text-to-image generation results.}
    }
    \label{fig:app_text_to_image_more}
\end{figure*}

\begin{figure*}[th]
 \noindent\rule{\textwidth}{0.4pt}
  \begin{minipage}[t]{0.52\textwidth}
 \centering
       \textbf{Prompt}
    \end{minipage}%
    \hfill
    \begin{minipage}[t]{0.45\textwidth}
     ~~~~~~w/o TD ~~~~~~~~~~~~~~~~~w/o SR ~~~~~~~~~~~~~~~\ours
    \end{minipage}
 \noindent\rule{\textwidth}{0.4pt}
    \begin{minipage}[t]{0.52\textwidth}
        \vspace{0pt}  
    A vibrant pink pig trots through a snowy landscape, \textcolor{blue}{a bright blue backpack strapped securely to its back}. The pig's thick coat contrasts with the soft white blanket of snow that covers the ground around it. As it moves, the blue backpack stands out against the pig's colorful hide and the winter scene, creating a striking visual amidst the serene, frost-covered backdrop.
    \end{minipage}%
    \hfill
    \begin{minipage}[t]{0.45\textwidth}
        \vspace{0pt}
        \raisebox{0pt}[\height][0pt]{\includegraphics[width=0.32\textwidth]{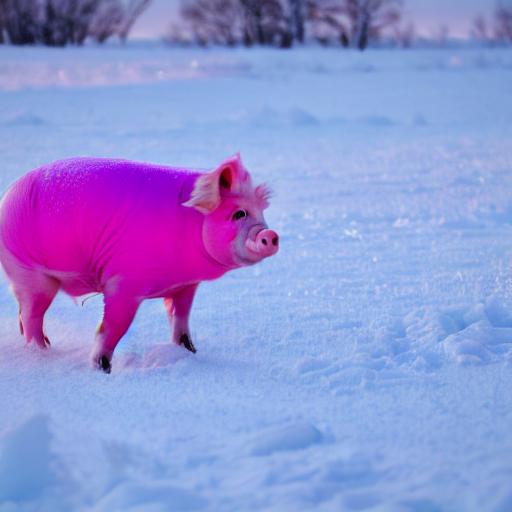}}%
        \hfill
        \raisebox{0pt}[\height][0pt]{\includegraphics[width=0.32\textwidth]{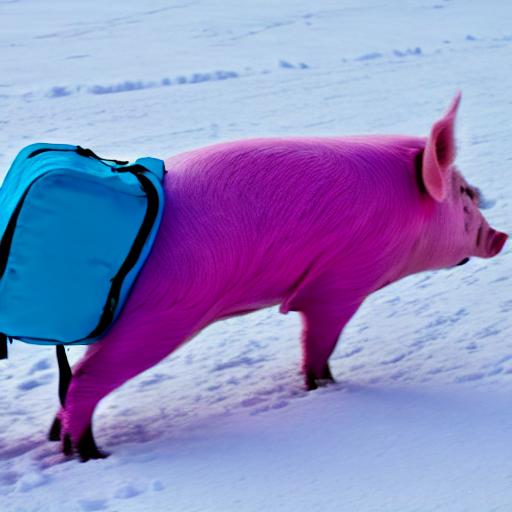}}%
        \hfill
        \raisebox{0pt}[\height][0pt]{\includegraphics[width=0.32\textwidth]{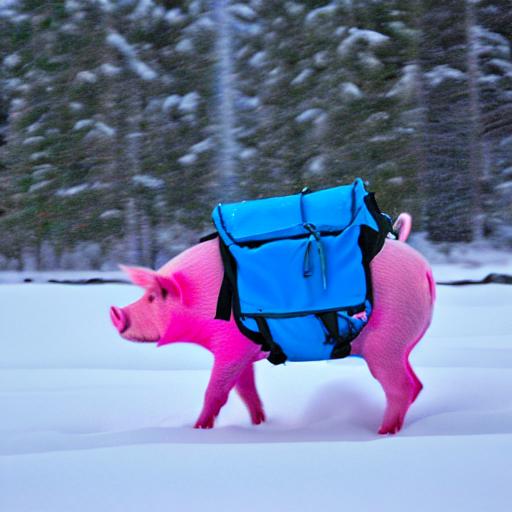}}
    \end{minipage}
     \noindent\rule{\textwidth}{0.4pt}
      \begin{minipage}[t]{0.52\textwidth}
        \vspace{0pt}  
    \end{minipage}%
    \hfill
    \begin{minipage}[t]{0.45\textwidth}
        \vspace{0pt}
        \raisebox{0pt}[\height][0pt]{\includegraphics[width=0.32\textwidth]{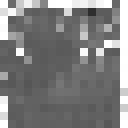}}%
        \hfill
        \raisebox{0pt}[\height][0pt]{\includegraphics[width=0.32\textwidth]{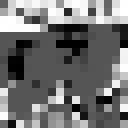}}%
        \hfill
        \raisebox{0pt}[\height][0pt]{\includegraphics[width=0.32\textwidth]{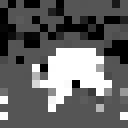}}
    \end{minipage}
     \noindent\rule{\textwidth}{0.4pt}
      \begin{minipage}[t]{0.52\textwidth}
        \vspace{0pt}  
    \end{minipage}%
    \hfill
    \begin{minipage}[t]{0.45\textwidth}
        \vspace{0pt}
        \raisebox{0pt}[\height][0pt]{\includegraphics[width=0.32\textwidth]{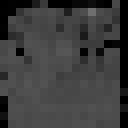}}%
        \hfill
        \raisebox{0pt}[\height][0pt]{\includegraphics[width=0.32\textwidth]{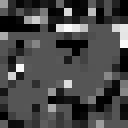}}%
        \hfill
        \raisebox{0pt}[\height][0pt]{\includegraphics[width=0.32\textwidth]{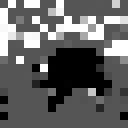}}
    \end{minipage}
    \noindent\rule{\textwidth}{0.4pt}
    \begin{minipage}[t]{0.52\textwidth}
        \vspace{0pt}  
  An outsized dolphin with a sleek, gray body glides through the blue waters, while a small, fluffy chicken with speckled brown and white feathers stands on the \textcolor{blue}{nearby sandy shore}, appearing \textcolor{blue}{diminutive} in comparison. The dolphin's fins cut through the water, creating gentle ripples, while \textcolor{blue}{the chicken pecks at the ground}, seemingly oblivious to the vast size difference. The stark contrast between the dolphin's smooth, aquatic grace and the chicken's terrestrial, feathered form is highlighted by their proximity to one another.
    \end{minipage}%
    \hfill
    \begin{minipage}[t]{0.45\textwidth}
        \vspace{0pt}
        \raisebox{0pt}[\height][0pt]{\includegraphics[width=0.32\textwidth]{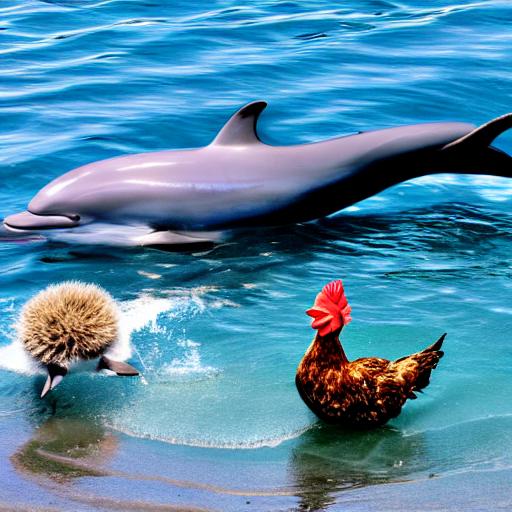}}%
        \hfill
        \raisebox{0pt}[\height][0pt]{\includegraphics[width=0.32\textwidth]{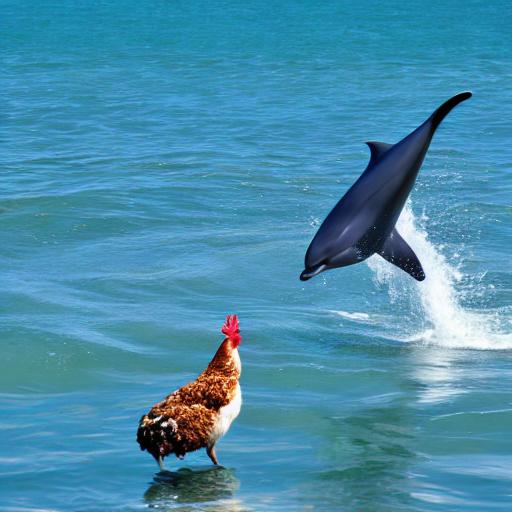}}%
        \hfill
        \raisebox{0pt}[\height][0pt]{\includegraphics[width=0.32\textwidth]{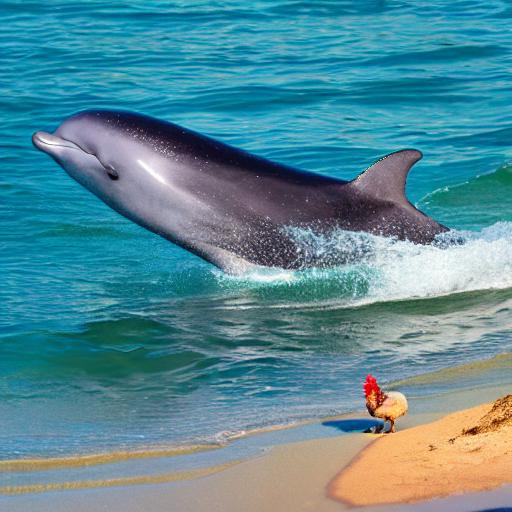}}
    \end{minipage}
     \noindent\rule{\textwidth}{0.4pt}
      \begin{minipage}[t]{0.52\textwidth}
        \vspace{0pt}  
    \end{minipage}%
    \hfill
    \begin{minipage}[t]{0.45\textwidth}
        \vspace{0pt}
        \raisebox{0pt}[\height][0pt]{\includegraphics[width=0.32\textwidth]{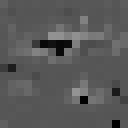}}%
        \hfill
        \raisebox{0pt}[\height][0pt]{\includegraphics[width=0.32\textwidth]{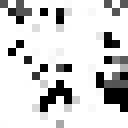}}%
        \hfill
        \raisebox{0pt}[\height][0pt]{\includegraphics[width=0.32\textwidth]{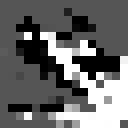}}
    \end{minipage}
     \noindent\rule{\textwidth}{0.4pt}
      \begin{minipage}[t]{0.52\textwidth}
        \vspace{0pt}  
    \end{minipage}%
    \hfill
    \begin{minipage}[t]{0.45\textwidth}
        \vspace{0pt}
        \raisebox{0pt}[\height][0pt]{\includegraphics[width=0.32\textwidth]{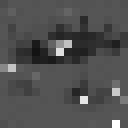}}%
        \hfill
        \raisebox{0pt}[\height][0pt]{\includegraphics[width=0.32\textwidth]{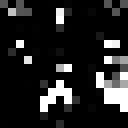}}%
        \hfill
        \raisebox{0pt}[\height][0pt]{\includegraphics[width=0.32\textwidth]{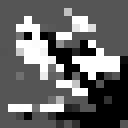}}
    \end{minipage}
    \caption{
        \textbf{More ablation studies.} 
        Without time-dependence (TD), the model fails to understand the relationship among the objects in the prompt. 
        Without sparsity regularization (SR), the influence of each prompt could be large, e.g., the attention map of local prompt 1 covers the pineapple and beers. Combining the two proposed designs, \ours generates images that accurately follow the complex text prompt.
    }
    \label{fig:app_ablation}
\end{figure*}

\subsection{More Empirical Understanding} \label{app:more_empirical_understanding}
While implicit models can be highly expressive, they can struggle with compositional generalization as many solutions might fit the training data but not generalize beyond. Our work introduces a theoretically motivated sparsity constraint (Eq.~\ref{eq:training_objective}) to select more generalizable solutions.
Following your suggestion, we've added fine-grained qualitative analysis in Fig.~\ref{fig:understanding_attention}. 
In Fig.~\ref{fig:understanding_attention}(a), our model attends to ``cat'' (L1) and ``sunglasses'' (L2) separately, and the baseline attends to all regions and omits ``sunglasses''.
Similarly, in Fig.~\ref{fig:understanding_attention}(b), our model, with sparse constraints focusing attention (L1 on ``bear'' and  L2 on ``cat''), renders both; the baseline's simultaneous generation misses ``cat''.
The analysis also highlights cases challenging to our model, such as the difficulty in decomposing words and printing the resultant letters correctly (L1 at 901 covers all letters simultaneously) in Fig.~\ref{fig:understanding_attention}(c).
While extreme sparsity can affect performance in dense interaction scenes (e.g., missing ``herb'' in Fig.~\ref{fig:understanding_attention}(d)), the model's superior performance over the baseline on examples here and all benchmarks confirms its robustness for these scenarios.

\begin{figure}[th]
    \centering
    \includegraphics[width=1.1\linewidth]{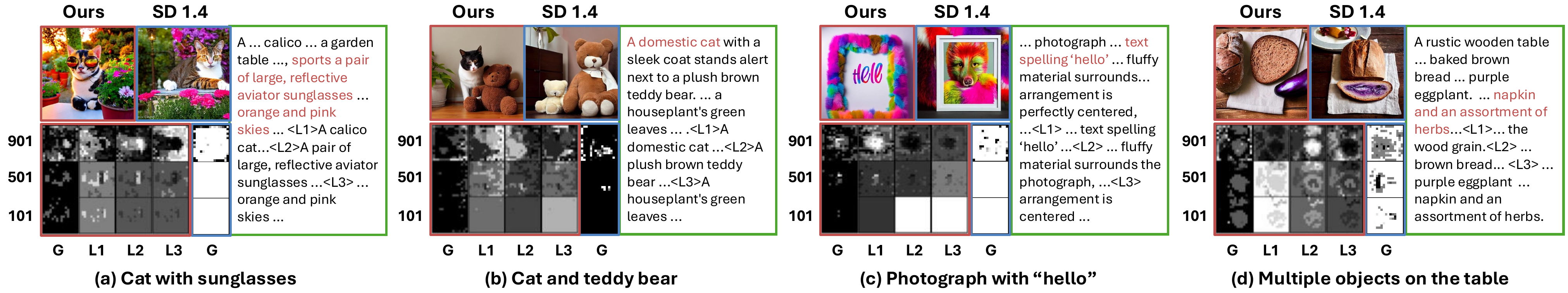}
    \caption{
        \textbf{Fine-grained comparison between our method and Stable Diffusion 1.4.} $G$ and $L_{i}$ indicate full caption and split captions (for our method), and indices denote diffusion steps ($901$ is closer to noise). White indicates high attention scores.
    }
    \label{fig:understanding_attention}
\end{figure}

\end{document}